\documentclass[pdflatex,sn-mathphys-num,iicol]{sn-jnl}% Math and Physical Sciences Numbered Reference Style
\usepackage{multirow}%
\usepackage{adjustbox}

\usepackage{amsmath,amssymb,amsfonts}%
\usepackage{bm}      % 支持向量粗体（如 \bm{x}）
\usepackage{amsthm}%
\usepackage[title]{appendix}%
\usepackage{xcolor}%
\usepackage{textcomp}%
\usepackage{manyfoot}%
\usepackage{booktabs}%
\usepackage{hyperref}
\usepackage{algorithm}%
\usepackage{algorithmic}

\usepackage{caption}
\usepackage{listings}%
\usepackage{xspace}%
\usepackage{array}    % 支持表格列宽调整
\usepackage{multicol}
\usepackage{multirow}
\usepackage{makecell}
\usepackage{graphicx}%
\usepackage{xcolor}
\usepackage{subfig}
\usepackage{adjustbox}

\makeatletter
\def\onedot{\ifx\@let@token.\else.\null\fi\xspace}
\makeatother
\def\eg{\emph{e.g}\onedot}

\newcommand{\red}[1]{\textcolor{red}{#1}}

\theoremstyle{thmstyleone}%
\newtheorem{theorem}{Theorem}%  meant for continuous numbers

\theoremstyle{thmstyletwo}%

\theoremstyle{thmstylethree}%
\newtheorem{definition}{Definition}%

\begin{document}

\title[Article Title]{\texttt{Holmes}: Towards Effective and Harmless Model Ownership Verification to Personalized Large Vision Models via Decoupling Common Features}

%%=============================================================%%
%% GivenName	-> \fnm{Joergen W.}
%% Particle	-> \spfx{van der} -> surname prefix
%% FamilyName	-> \sur{Ploeg}
%% Suffix	-> \sfx{IV}
%% \author*[1,2]{\fnm{Joergen W.} \spfx{van der} \sur{Ploeg} 
%%  \sfx{IV}}\email{iauthor@gmail.com}
%%=============================================================%%

\author[1,3]{\fnm{Linghui} \sur{Zhu}}\email{zlh20@mails.tsinghua.edu.cn}
%\equalcont{These authors contributed equally to this work.}

\author*[2]{\fnm{Yiming} \sur{Li}}\email{liyiming.tech@gmail.com}

\author[3]{\fnm{Haiqin} \sur{Weng}}\email{haiqin.wenghaiqin@antgroup.com}

\author[3]{\fnm{Yan} \sur{Liu}}\email{bencao.ly@antgroup.com}

\author[2]{\fnm{Tianwei} \sur{Zhang}}\email{tianwei.zhang@ntu.edu.sg}

\author[1,4]{\fnm{Shu-Tao} \sur{Xia}}\email{xiast@sz.tsinghua.edu.cn}

\author[1]{\fnm{Zhi} \sur{Wang}}\email{wangzhi@sz.tsinghua.edu.cn}

% \affil*[1]{\orgdiv{Department}, \orgname{Organization}, \orgaddress{\street{Street}, \city{City}, \postcode{100190}, \state{State}, \country{Country}}}
\affil[1]{ \orgdiv{Tsinghua Shenzhen International Graduate School} \orgname{Tsinghua University}, \orgaddress{ \city{Shenzhen}, \postcode{518055}, \country{China}}}

\affil[2]{ \orgname{Nanyang Technological University}, \orgaddress{ \postcode{639798}, \country{Singapore}}}

\affil[3]{ \orgname{Ant Group}, \orgaddress{ \city{Hangzhou}, \postcode{310023}, \country{China}}}

\affil[4]{ \orgname{Peng Cheng Laboratory}, \orgaddress{ \city{Shenzhen}, \postcode{518055}, \country{China}}}

%\affil[5]{ \orgname{Tencent}, \orgaddress{ \city{Shenzhen}, \postcode{518054}, \country{China}}}
% \affil[3]{\orgdiv{Department}, \orgname{Organization}, \orgaddress{\street{Street}, \city{City}, \postcode{610101}, \state{State}, \country{Country}}}

%%==================================%%
%% Sample for unstructured abstract %%
%%==================================%%

\abstract{
Large vision models (LVMs) achieve remarkable performance in various downstream tasks, primarily by personalizing pre-trained models through fine-tuning with private and valuable local data, which makes the personalized model a valuable intellectual property. Similar to the era of traditional DNNs, model stealing attacks also pose significant risks to LVMs. However, this paper reveals that most existing defense methods (developed for traditional DNNs), typically designed for models trained from scratch, either introduce additional security risks, are prone to misjudgment, or are even ineffective for fine-tuned models. To alleviate these problems, this paper proposes a harmless model ownership verification method for personalized LVMs by decoupling similar common features. In general, our method consists of three main stages. In the first stage, we create shadow models that retain common features of the victim model while disrupting dataset-specific features. We represent the dataset-specific features of the victim model by computing the output differences between the shadow and victim models, without altering the victim model or its training process. After that, a meta-classifier is trained to identify stolen models by determining whether suspicious models contain the dataset-specific features of the victim. In the third stage, we conduct model ownership verification by hypothesis test to mitigate randomness and enhance robustness. Extensive experiments on benchmark datasets verify the effectiveness of the proposed method in detecting different types of model stealing simultaneously. Our codes are available at https://github.com/zlh-thu/Holmes.
}


%%\pacs[JEL Classification]{D8, H51}

%%\pacs[MSC Classification]{35A01, 65L10, 65L12, 65L20, 65L70}

\maketitle
% \clearpage
\section{Introduction}\label{sec:intro}

Pre-trained large vision models (LVMs) are widely used across various applications \cite{VaswaniSPUJGKP17, radford2021learning,gu2024rwkv}. The dominant development paradigm for specific downstream tasks involves personalizing LVMs via fine-tuning on local data. This approach results in a well-performing task-oriented model at a small training cost since the foundation model has already learned a high-quality representation from a large amount of data \cite{wei2023improving,wang2023image,li2025rethinking}. Generally, these personalized models have significant commercial value, and their fine-tuning process involves private and valuable data, making them critical intellectual property for the owners.

\begin{figure}[!t]
    %\vspace{-2.0em}
    \centering    \includegraphics[width=0.45\textwidth]{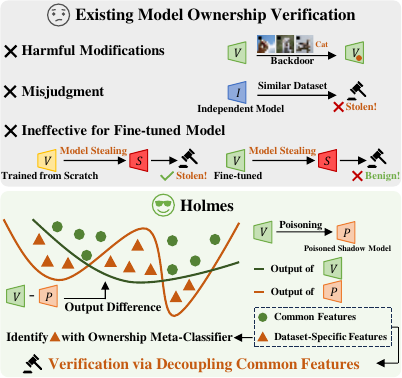}
    %\vspace{-0.3em}
\caption{Limitations of existing model ownership verification and our solution via decoupling common features. Existing model ownership verification methods suffer from three main limitations: (1) Introducing harmful modifications ($e.g.$, backdoor) that compromise model reliability; (2) Vulnerability to misjudgments when models share common features; (3) Ineffectiveness for fine-tuned models since most defense methods are primarily designed for models trained from scratch. We propose \texttt{Holmes} to systematically alleviate these limitations via feature decoupling: (1) Ensuring harmlessness through non-invasive verification that makes use of shadow models, avoiding modifications to the original victim model; (2) Achieving robustness through an ownership meta-classifier to identify dataset-specific features; (3) Ensuring reliability for fine-tuned models by leveraging the learned inherent dataset-specific features, eliminating the need for external feature embedding or artificial modifications during fine-tuning.
}
\label{fig:intro}
\vspace{-1.0em}
\end{figure}

% model parameters 

However, recent studies revealed that the adversary could `steal' (personalized) models even without access to their training samples or model parameters \cite{tramer2016stealing, chandrasekaran2020exploring, carlini2024stealing} via model stealing attacks. For example, the adversary might use the victim model to label an unlabeled dataset or apply knowledge distillation to reproduce the victim model \cite{gou2021knowledge, li2023curriculum,sun2024logit}. Such attacks pose a severe threat to model owners, as these attacks require significantly fewer resources than directly conducting personalization ($e.g.$, fewer computational resources or less training data).

Currently, model ownership verification (MOV) is an important audit strategy to protect the model copyright. It aims to identify whether a suspicious third-party model is stolen from the victim. Existing MOV methods usually fall into two categories: model watermarking and model fingerprinting. Specifically, the model watermarking approaches \cite{ya2023towards,zhu2024reliable,wang2025sleepermark} mostly rely on embedding defender-specified misclassification behaviors, especially via backdoor attacks \cite{li2022backdoor, gao2024backdoor, liang2025vl}, into the victim model. However, these methods inevitably introduce potential vulnerabilities, compromising model reliability and limiting practical deployment. Most recently, to address this limitation, MOVE \cite{li2025move} proposed a harmless watermarking scheme that avoids embedding misclassification behaviors. Specifically, MOVE defined harmlessness as the watermarked model should have similar prediction behaviors to the one without any watermark ($i.e.$, benign model trained with the benign dataset). Accordingly, instead of leveraging misclassification behaviors, MOVE modified a subset of training samples using style transfer without manipulating their labels to embed external features as dataset-specific features into the training set. By identifying whether the model has learned the external features, MOVE can provide effective model ownership verification. In particular, since the modifications to the training set do not alter the labels of the affected data, this method remains harmless. However, it is primarily designed for models trained from scratch. As we will reveal in Section \ref{sec:limitation}, MOVE is ineffective for fine-tuned models because the required external features cannot be adequately learned during fine-tuning due to interference from extensive `robust features' learned by the foundation model.

In contrast, model fingerprinting methods \cite{ peng2022fingerprinting, yang2022metafinger, shao2025sok} verify ownership by detecting whether the suspicious models retain the inherent features naturally learned by the victim model from its training data. For example, dataset inference \cite{maini2021dataset} uses the decision boundary of the victim model as a representation of its learned features. Unlike watermarking, fingerprinting does not modify the model's training process, therefore naturally ensuring harmlessness. However, these methods are primarily designed for simple model stealing attacks ($e.g.$, direct copying or fine-tuning), whereas exhibiting limited robustness against advanced stealing methods \cite{maini2021dataset, li2025move}. Besides, as we will show in Section \ref{sec:limitation}, relying solely on inherent features may lead to misjudgment when the victim and suspicious models are trained on similar datasets, as they may learn similar `common features'. These limitations raise an intriguing question: \emph{Can we design effective and harmless MOV methods for personalized LVMs?}

Fortunately, the answer to the above question is positive, although its solution is non-trivial. Motivated by these insights, we aim to reduce misjudgments in model fingerprinting by decoupling common features and isolating dataset-specific characteristics from the model's inherent representations. We then exploit these decoupled inherent yet dataset-specific features for fingerprinting: the inherent component ensures effective learnability, while the dataset-specific component reduces the likelihood of misjudgment. In contrast to prior fingerprinting approaches that mainly optimize the representation of inherent features~\cite{peng2022fingerprinting, yang2022metafinger, shao2025sok}, this paper advances a new fingerprinting paradigm that explicitly suppresses the interference of shared common features, which are likely a primary source of misjudgments. In general, our method (dubbed `\texttt{Holmes}') consists of three main stages, including \textbf{(1)} creating shadow models, \textbf{(2)} training ownership meta-classifier, and \textbf{(3)} ownership verification with hypothesis test. Specifically, in the first stage, we create two shadow models, including one poisoned shadow model and one benign shadow model, to decouple the common features and distinguish the learned dataset-specific features of the victim model. In general, the poisoned shadow model retains common features learned by the victim model, while its dataset-specific features are disrupted. As such, we can indirectly capture the dataset-specific features while decoupling the common features via the output difference between these two models. To achieve it, we poison the most well-learned personalization samples (with the lowest loss values) to disrupt dataset-specific features via the label-inconsistent backdoor attack ($e.g.$, BadNets \cite{gu2019badnets}). Intuitively, we choose these samples since they correspond to the model's propensity to memorize dataset-specific patterns (instead of common features) \cite{ilyas2019adversarial}. We exploit the (label-inconsistent) backdoor attack since it allows a high performance of poisoned models in predicting benign testing samples, thus preserving common features. The benign shadow model is introduced for reference to further mitigate false positives from independent models with related dataset-specific features. In general, this model exhibits dataset-specific features that are similar but distinct from those of the victim model. Accordingly, the output difference between the benign and poisoned shadow models reveals dataset-specific features that are distinct from those of the victim model. To achieve this model, we fine-tune the same foundation model on personalization samples excluding those that were well-learned by the victim model. In the second stage, we train a binary ownership meta-classifier for identifying dataset-specific features of the victim model, based on two output differences: \textbf{(1)} the output differences between the victim model and the poisoned shadow model (for capturing the victim model's dataset-specific features) and \textbf{(2)} the output differences between the benign shadow model and the poisoned shadow model (for capturing features unrelated to the victim model's dataset-specific features). In the third stage, we conduct model ownership verification via hypothesis test to mitigate randomness and further enhance robustness. Throughout all stages of our workflow, neither the (released) victim model nor its training process is modified. The proposed method is therefore intrinsically harmless.

The main contribution of this work is four-fold: \textbf{(1)} We revisit the ownership verification methods and reveal the limitations of existing methods for personalized LVMs. \textbf{(2)} We propose a new fingerprinting paradigm that suppresses the interference of shared common features, addressing a primary cause of misjudgments in ownership verification. \textbf{(3)} We develop an effective and intrinsically harmless model fingerprinting method based on feature decoupling, instantiated through the construction of shadow models and hypothesis-test-based verification. \textbf{(4)} We validate the effectiveness of our method on benchmark datasets for simultaneously detecting multiple types of model stealing and further discuss its potential extensions to generative tasks such as image captioning.

%%%%%%%%%%

\section{Related Work}
\label{sec:related-work}
\subsection{Large Vision Model}

By learning high-quality representations from extensive datasets, large pre-trained vision models (LVMs) have demonstrated remarkable capabilities in comprehending image content and extracting rich semantic information \cite{ShenLTBRCYK22,wang2023image,gu2024rwkv}. These models excel across a wide range of computer vision tasks, including image classification \cite{mauricio2023comparing,gao2024clip, al2025analysing}, transfer learning \cite{wei2023improving,wang2023image,gu2024rwkv}, object detection \cite{kirillov2023segment, cai2024crowd, wu2025weakly}, and semantic segmentation \cite{wang2025reclip++, zhu2025weakclip, xie2025clims++}. Among the most notable LVMs, CLIP (Contrastive Language-Image Pre-Training) \cite{radford2021learning} stands out as a cornerstone in the field. CLIP leverages contrastive learning to align textual and visual information in a shared embedding space, enabling it to perform image classification \cite{mauricio2023comparing,gao2024clip, al2025analysing} and transfer learning \cite{rasheed2023fine} with exceptional accuracy. Pre-trained on a massive dataset of image-text pairs, CLIP is widely adapted as a foundation model for various downstream tasks \cite{mokady2021clipcap, lin2024open, sun2025fairness}. Other prominent LVMs include BEiT (Bidirectional Encoder Representation from Image Transformers) \cite{bao2021beit}, which also achieves outstanding performance in various vision tasks \cite{wang2023image}.

% start from here

Currently, the most prevalent application paradigm for LVMs involves personalization on downstream tasks via fine-tuning \cite{wei2023improving,wang2023image, gu2024rwkv}. This process typically begins with a pre-trained foundation model, which has already learned high-quality generalizable features from a large dataset. The foundation model is then adapted to specific tasks by fine-tuning on task-specific datasets. For example, in medical imaging, a pre-trained LVM can be fine-tuned on a dataset of X-ray images to significantly improve disease diagnosis accuracy \cite{lin2024open, sun2025fairness}. Fine-tuning not only leverages the pre-trained model's robust feature extraction capabilities but also significantly reduces computational costs and training time compared to training models from scratch. However, as many pre-trained LVMs are non-public and task-specified datasets may involve highly private samples, the personalized models are a valuable intellectual property for the owners.

\subsection{Model Stealing}
\label{sec:model-stealing}
Model stealing refers to the unauthorized replication of a victim model's functionality, thereby compromising the model owner's intellectual property. Depending on how the adversary interacts with the victim model, existing attacks can be categorized into the following types:

\vspace{0.5em}
\noindent\textbf{Parameter-Based Stealing.} This category involves direct use of the victim model's parameters, requiring full access to the model through either public release or unauthorized means. The adversary may directly copy the model or fine-tune it to produce a stolen model \cite{maini2021dataset,li2025move}.

\vspace{0.5em}
\noindent\textbf{Distillation-Based Stealing.} In this setting, the adversary does not directly access or modify the victim model's parameters. Instead, the victim model is treated as a teacher model, and a surrogate (student) model is trained to mimic its functionality via knowledge distillation. These attacks are further divided based on the adversary's data access: \textbf{(1) Training Data-Assisted Attacks}: Given access to the training dataset, the adversary can employ knowledge distillation \cite{li2023curriculum} to train a surrogate model that mimics the victim's input-output behavior. \textbf{(2) Query-Only Attacks}: Without access to any training samples, the adversary can still conduct model stealing via data-free distillation \cite{yu2023data, liu2024small, tran2024nayer}, hard distillation attack \cite{papernot2017practical,jagielski2020high}, and soft distillation attack \cite{orekondy2019knockoff, cai2025llava}. Specifically, the data-free distillation method employs an additional generative model to generate data samples for the distillation process. Hard distillation attacks utilize the victim model to label substitute (unlabeled) samples, which are then used to train their substitute model. Soft distillation attacks typically obtain a substitute model by minimizing the distance between its predicted logits of the substitute samples and those generated by the victim model. In particular, some approaches, such as the Distilled Fine-Tuning (DFT) \cite{cai2025llava}, extend beyond using logits and exploit intermediate representations (\eg, vision tokens). While this enables more effective knowledge transfer, it comes at the cost of requiring access to the model's internal features.

\subsection{Defense against Model Stealing}
\label{sec:background_defenses}

Existing defenses against model stealing attacks can be broadly categorized into two main types: active defenses and verification-based defenses.

\vspace{0.3em}
\noindent \textbf{Active Defenses.} Active defenses aim to increase the cost or reduce the efficacy of model stealing attacks by increasing the number of queries or reducing the accuracy of the stolen model. Common strategies include introducing perturbations to model outputs or restricting the output form of the victim model. For examples, defenders may apply output probability rounding \cite{tramer2016stealing}, add noise to logit outputs \cite{lee2018defending, zhang2023apmsa}, or only return the most confident label instead of the whole output vector \cite{orekondy2019knockoff}. These methods disrupt the adversary's ability to accurately mimic the behavior of the victim model, thereby defending against the model stealing attacks. However, these defenses may significantly degrade the performance of the victim models and may even be ineffective against adaptive stealing attacks \cite{jia2021entangled,maini2021dataset,li2025move}.

\vspace{0.3em}
\noindent \textbf{Verification-Based Defenses.} In general, model ownership verification aims to verify whether a suspicious model is stolen from the victim model. Existing verification-based methods can be divided into two main categories, including \emph{model watermarking} and \emph{model fingerprinting}.

\vspace{0.3em}
\noindent \textbf{(1) Model Watermarking.} Most model watermarking methods conduct verification by modifying the victim model to induce defender-specified misclassification behaviors \cite{jia2021entangled, wang2023free, zhu2024reliable}. For example, backdoor-based verification embeds specific trigger-label pairs into the training data, which lie outside the distribution of the original data. During verification, a model is considered stolen if it misclassifies the triggered samples to the target label, similar to the victim model's behavior. However, these methods introduce latent vulnerabilities, such as hidden backdoors, which can be exploited maliciously, compromising the model's functionality and reliability. Additionally, they may be vulnerable to backdoor removal techniques \cite{wang2019neural,chen2024anti,zhu2024neuralsanitizer}, which can erase the watermarks \cite{li2025move}.

To the best of our knowledge, MOVE \cite{li2025move} represents the state-of-the-art (SOTA) in harmless model watermarking. Unlike backdoor-based methods, MOVE embeds external features as dataset-specific features into the training set by modifying a subset of training samples using style transfer. Thus, MOVE can identify the stolen model by verifying whether the model has learned the external features. Crucially, these dataset modifications do not alter the labels of the affected data, ensuring the method's harmlessness. Specifically, considering a $K$-classification problem, the defender pre-defines a style transformer $T: \mathcal{X} \times \mathcal{X} \rightarrow \mathcal{X}$, where $\bm{x}_s$ is a defender-specified \emph{style image}. The defender then randomly selects $\gamma \%$ samples ($i.e.$, $\mathcal{D}_s$) from the training set $\mathcal{D}$ to generate their transformed version $\mathcal{D}_t = \{ (\bm{x}', y) |\bm{x}' = T(\bm{x}, \bm{x}_s), (\bm{x}, y) \in \mathcal{D}_s \}$. The victim model $V$ is trained with $\mathcal{D}_t \cup \mathcal{D}_b$, where $\mathcal{D}_b \triangleq \mathcal{D} \backslash \mathcal{D}_s$. Consequently, the output difference between a suspicious model $S$, $i.e.,$ $S(\bm{x}') - S(\bm{x})$, serves as an indicator of potential model theft. However, MOVE is primarily designed for scenarios where the victim model is trained from scratch. As we reveal in Section \ref{sec:limitation}, MOVE becomes ineffective for personalized models, as these models struggle to learn external, dataset-specific features during the fine-tuning process. This is primarily because the pre-existing knowledge in the foundation model inhibits the learning of new external features.

\vspace{0.3em}
\noindent \textbf{(2) Model Fingerprinting.} Unlike model watermarking methods, model fingerprinting does not require any modifications to the victim model's training process, ensuring a harmless verification process. Most existing methods \cite{maini2021dataset, peng2022fingerprinting, yang2022metafinger} verify ownership by leveraging the inherent features learned by the victim model from its original training data. For example, dataset inference \cite{maini2021dataset} is a classical model fingerprinting method that uses the decision boundary of the victim model as a representation of the learned features. Specifically, for each sample $(\bm{x}, y)$ in the training set, dataset inference defines the minimum distance to each class ($i.e.$, $\bm{\delta}=(\bm{\delta}_1, \cdots, \bm{\delta}_K)$) as the decision boundary of the victim model $V$. The minimum distance $\bm{\delta}_t$ to each class $t$ is generated via: 
\begin{equation}
    \min_{\bm{\delta}_t} d(\bm{x}, \bm{x}+\bm{\delta}_t), s.t., V(\bm{x}+\bm{\delta}_t) = t,
\end{equation}
where $d(\cdot)$ is a distance metric ($e.g.$, $\ell^\infty$ norm). The defender then selects samples from private and public datasets, using $\bm{\delta}$ as feature embeddings to train a binary meta-classifier $C$, where $C(\bm{\delta}) \in [0,1]$ estimates the probability of a sample belonging to the private dataset. To identify stolen models, the defender creates equal-sized sample vectors from private and public samples and conducts a hypothesis test based on the trained $C$. If the confidence scores of private samples are significantly higher than those of public samples, the suspicious model is treated as containing the features from private samples, thereby being identified as a stolen model.

Follow-up works \cite{peng2022fingerprinting,yang2022metafinger} aim to improve the characterization of learned inherent features. For example, Peng \emph{et al.} introduced universal adversarial perturbations (UAPs) \cite{peng2022fingerprinting} to capture the decision boundary of the victim model, which serves as a representation of the learned inherent features. Metafinger \cite{yang2022metafinger} leverages meta-training to capture the decision regions of the victim model. Metafinger generates multiple shadow models as meta-data and optimizes specific images through meta-training, ensuring that only models derived from the victim model can correctly classify them. However, model fingerprinting methods are primarily designed for simple model stealing attacks, $e.g.$, direct copying or fine-tuning. These methods exhibit limited robustness against advanced model stealing methods \cite{li2022defending,li2025move}, $e.g.$, knowledge distillation \cite{li2023curriculum,tran2024nayer}, making them insufficient for protecting models in complex adversarial scenarios. Additionally, as we reveal in Section \ref{sec:limitation}, verification based solely on inherent features may easily lead to misjudgments, especially when the training sets of the suspicious model and the victim model have similar distributions \cite{li2022defending,li2025move}. This is primarily because different models may learn similar common features when their training sets share certain distributional similarities.

In conclusion, existing defenses still have significant limitations. How to design effective and harmless ownership verification methods for fine-tuned personalized LVMs remains an important open question and warrants further exploration.

\section{Revisiting Harmless Model Ownership Verification}
\label{sec:limitation}
This section revisits harmless model ownership verification in the context of fine-tuned personalized models, covering both \emph{harmless model watermarking} and \emph{model fingerprinting}. We take on two representative approaches as study cases: MOVE \cite{li2025move} for harmless watermarking and dataset inference \cite{maini2021dataset} for model fingerprinting. Our analysis specifically focuses on their limitations when applied to fine-tuned personalized models.

\subsection{Preliminaries}

\vspace{0.3em}
\noindent \textbf{Problem Formulation.} This paper focuses on defending against model stealing in image classification tasks via model ownership verification. Specifically, given a suspicious model $S$, the defender intends to determine whether it has been stolen from the victim model $V$. Let $\mathcal{D} = \{ (\bm{x}_i, y_i) \}_{i=1}^{N}$ be the training set used to fine-tune the pre-trained foundation model $F$. The victim model $V$ is obtained by:
\begin{equation}
    V = \arg \min_{\bm{\theta}} \sum_{(\bm{x}, y) \in \mathcal{D}} \mathcal{L}(F_{\bm{\theta}}(\bm{x}), y),
\end{equation}
where $\mathcal{L}(\cdot)$ is the loss function ($i.e.$, cross-entropy), and $\bm{\theta}$ represents the parameters of the foundation model $F$. This fine-tuning process adapts $F$ to the specific downstream task defined by $\mathcal{D}$, rather than trains the model from scratch.

For a model ownership verification method to be effective in practice, it must satisfy two key requirements: \emph{effectiveness} and \emph{harmlessness}. Effectiveness means that the defense must accurately determine whether the suspicious model has been stolen from the victim, regardless of the model stealing technique employed. At the same time, it should not misclassify independently trained or fine-tuned benign models as stolen models. Harmlessness ensures that the defense does not introduce additional security risks, such as backdoors. Furthermore, if the method modifies the training set, harmlessness implies that models trained or fine-tuned with the modified dataset should exhibit behavior similar to those trained with the original dataset. In particular, this paper adopts a stricter harmlessness criterion, requiring that the defense must not modify the victim model's parameters in any form, including direct alteration, additional training, or fine-tuning.

\vspace{0.3em}
\noindent \textbf{Threat Model.} In this paper, we consider the black-box setting for model ownership verification. Specifically, we assume that the defenders can only query and obtain the predicted probabilities (logit outputs) from the suspicious model, with no access to model source files or parameters, intermediate computational results ($e.g.$, gradients), or any information about the model stealing process.

\subsection{Revisiting Harmless Model Watermarking}
As illustrated in Section \ref{sec:background_defenses}, MOVE \cite{li2025move} relies on a latent assumption that the model can sufficiently learn the embedded external dataset-specific features via modifying a few training samples with style transfer. However, since MOVE is designed for victim models trained from scratch, this assumption may not hold when fine-tuning is applied. Therefore, this method may fail to provide effective ownership verification. In this section, we demonstrate this limitation.

\vspace{0.3em}
\noindent \textbf{Settings.} We evaluate MOVE's effectiveness on CIFAR-10 \cite{krizhevsky2009learning} using ResNet18 \cite{he2016deep} under three settings: fine-tuning, unlearning-based fine-tuning, and training from scratch. Specifically, let $\mathcal{D} = \{ (\bm{x}_i, y_i) \}_{i=1}^{N}$ denote the benign training set, $\bm{x}_s$ be a defender-specified \emph{style image}, and $T: \mathcal{X} \times \mathcal{X} \rightarrow \mathcal{X}$ be a pre-trained style transformer. Following the setting of MOVE \cite{li2025move}, we randomly select $10\%$ samples ($i.e.$, $\mathcal{D}_s$) from $\mathcal{D}$ to generate their transformed versions $\mathcal{D}_t = \{ (\bm{x}', y) |\bm{x}' = T(\bm{x}, \bm{x}_s), (\bm{x}, y) \in \mathcal{D}_s \}$. Under the original assumption of MOVE, the victim model should sufficiently learn the external features from $\mathcal{D}_t$ by training or fine-tuning on $\mathcal{D}_b \cup \mathcal{D}_t$. To evaluate this assumption under fine-tuning process, three ResNet18 models are then prepared: \textbf{(1)} ResNet18-FT: Fine-tuned from a pre-trained ResNet18 on $\mathcal{D}_b \cup \mathcal{D}_t$, where the pre-trained ResNet18 is trained with ImageNet \cite{deng2009imagenet} dataset and $\mathcal{D}_b \triangleq \mathcal{D} \backslash \mathcal{D}_t$. \textbf{(2)} ResNet18-UL: First unlearned on randomly relabeled $\mathcal{D}$ for 10 epochs before fine-tuning on $\mathcal{D}_b \cup \mathcal{D}_t$. \textbf{(3)} ResNet18-S: Trained from scratch on $\mathcal{D}_b \cup \mathcal{D}_t$.

We evaluate the effectiveness of MOVE on ResNet18-FT, ResNet18-UL, and ResNet18-S under the simplest form of model stealing, where the victim model is directly copied and fine-tuned to create the stolen models. Following the setting in MOVE \cite{li2025move}, we adopt the p-value as the evaluation metric, where \emph{the smaller the p-value, the more confident that MOVE believes the model stealing happened}. We mark failed verification cases ($i.e.$, p-value $>0.01$ for stolen models) in red. We also test the accuracy of each model on the transformed dataset $\mathcal{D}_t$ and the unmodified CIFAR-10 test dataset to assess whether they have learned the external features in $\mathcal{D}_t$.

\begin{table}[t]
\centering
\vspace{-1.8em}
\caption{Failure of MOVE on fine-tuned models.}
%\vspace{-0.2em}
\setlength{\tabcolsep}{1.4pt}
\begin{tabular}{c|ccc}
\toprule  
Model Stealing $\downarrow$ & ResNet18-FT&ResNet18-UL& ResNet18-S\\ \hline
Direct-copy  &  \red{0.55} &$10^{-4}$& $10^{-15}$ \\
Fine-tuning  & \red{0.48} &$10^{-3}$& $10^{-10}$  \\
\bottomrule
\end{tabular}
%\vspace{-0.5em}
\label{table:move_p_val}
\vspace{-3.0em}
\end{table}

\begin{table}[t]
\centering
\caption{Accuracy on $\mathcal{D}_t$ and the benign set.}
\setlength{\tabcolsep}{3pt}
%\vspace{-0.6em}
\begin{tabular}{c|ccc}
\toprule  
Test Set $\downarrow$ & ResNet18-FT &ResNet18-UL& ResNet18-S \\ \hline
$\mathcal{D}_t$ & 17.54\%& 42.45\% & 88.17\% \\
Benign Set  & 93.16\%& 87.50\% & 92.07\%\\
\bottomrule
\end{tabular}
%\vspace{-0.8em}
\label{table:move_acc}
\vspace{-1.8em}
\end{table}

\vspace{0.3em}
\noindent \textbf{Results.} As shown in Table \ref{table:move_p_val}, MOVE successfully detects model stealing for ResNet18-UL and ResNet18-S, while it fails to verify ownership for the fine-tuned model ResNet18-FT. This performance gap correlates with their learning capabilities of the external features. As shown in Table \ref{table:move_acc}, ResNet18-FT achieves only 17.54\% accuracy on the transformed dataset $\mathcal{D}_t$, compared to 42.45\% for ResNet18-UL and 88.17\% for ResNet18-S. These findings reveal two key insights: First, fine-tuning process lacks sufficient capacity for learning external dataset-specific features, likely due to interference from the pre-trained model's existing knowledge. Second, the unlearning process in ResNet18-UL appears to mitigate this interference, enabling better acquisition of external features. The results collectively demonstrate that MOVE's verification mechanism, which relies on learning external features as dataset-specific features, is ineffective for fine-tuning scenarios due to interference from learned knowledge in the pre-trained foundation model.

\begin{table}[!t]
\centering
\vspace{-1.9em}
\caption{Misjudgments of DI on fine-tuned models.}
\setlength{\tabcolsep}{1.3pt}
%\vspace{-0.7em}
\begin{tabular}{c|ccc}
\toprule  
 & ResNet18-FT-$\mathcal{D}_r$ & VGG16-FT-$\mathcal{D}_l$ & VGG16-FT-$\mathcal{D}_l'$ \\ \hline
Accuracy & 92.45\% &92.94\%	&63.53\%\\
p-value  & \red{$10^{-48}$}	& \red{$10^{-28}$} &\red{$10^{-21}$}\\
\bottomrule
\end{tabular}
\label{table:dataset_inference_mis}
\vspace{-2em}
\end{table}

\subsection{Revisiting Model Fingerprinting}
As discussed in Section \ref{sec:background_defenses}, model fingerprinting methods are primarily designed for simple model stealing attacks and show limited robustness against advanced stealing methods \cite{maini2021dataset, li2025move}, such as knowledge distillation \cite{li2023curriculum,sun2024logit}. Experiment results demonstrating these limitations are presented in Section \ref{sec:main_experiments}. In this section, we focus specifically on the false positive problem in model fingerprinting, where benign models may be incorrectly identified as stolen.

Specifically, we take dataset inference (DI) \cite{maini2021dataset}, a representative fingerprinting method, as the study case. As demonstrated in previous studies \cite{li2022defending, li2025move}, when verifying the ownership of models trained from scratch, DI easily leads to misjudgment when the training set of the suspicious models has a similar distribution to the victim model training set. As illustrated in the following experiments, this limitation also applies to the fine-tuned victim models.

\vspace{0.3em}
\noindent \textbf{Settings.} We conduct experiments on the CIFAR-10 \cite{krizhevsky2009learning} dataset with VGG16 \cite{simonyan2014very} and ResNet18 \cite{he2016deep}. Specifically, we randomly split the original training set $\mathcal{D}$ into two disjoint subsets, $\mathcal{D}_l$ and $\mathcal{D}_r$. We use the pre-trained VGG16 and ResNet18 models, both trained on the ImageNet \cite{deng2009imagenet}, as the foundation models. Three different models are then prepared: (1) VGG16-FT-$\mathcal{D}_l$: Fine-tuned from the pre-trained VGG16 on $\mathcal{D}_l$. (2) ResNet18-FT-$\mathcal{D}_r$: Fine-tuned from the pre-trained ResNet18 on $\mathcal{D}_r$. (3) VGG16-FT-$\mathcal{D}_l'$: Fine-tuned from the pre-trained VGG16 on a noisy dataset $\mathcal{D}_l' \triangleq \{(\bm{x}', y)|\bm{x}' = \bm{x} + \mathcal{N}(0, 16), (\bm{x}, y) \in \mathcal{D}_l\}$.

In the model ownership verification process, we verify whether VGG16-FT-$\mathcal{D}_l$ and VGG16-FT-$\mathcal{D}_l'$ are stolen from ResNet18-FT-$\mathcal{D}_r$, and whether ResNet18-FT-$\mathcal{D}_r$ is stolen from VGG16-FT-$\mathcal{D}_l$. Following the setting of dataset inference \cite{maini2021dataset}, we adopt the p-value as the evaluation metric, where \emph{the smaller the p-value, the more confident that dataset inference believes the model stealing happened}. It is worth noting that $\mathcal{D}_l$, $\mathcal{D}_r$, and $\mathcal{D}_l'$ are highly similar in distribution but not identical, and the three models are not involved in any model stealing between each other.

\vspace{0.3em}
\noindent \textbf{Results.} As shown in Table \ref{table:dataset_inference_mis}, DI misjudges all cases. Although ResNet18-FT-$\mathcal{D}_r$, VGG16-FT-$\mathcal{D}_l$, and VGG16-FT-$\mathcal{D}_l'$ are fine-tuned on different datasets and should not be considered as stolen, the method incorrectly claims ownership between them. These misjudgments are probably because the fine-tuned models learn similar common features from $\mathcal{D}_r$, $\mathcal{D}_l$, and $\mathcal{D}_l'$.

\section{The Proposed Method}
\label{sec:method}

Based on the understandings in Section \ref{sec:limitation}, where we revealed that existing harmless MOV methods either fail for fine-tuned models due to interference from pre-trained knowledge or suffer from misjudgments caused by shared common features, we propose \texttt{Holmes}, which exploits inherent yet dataset-specific features for verification via decoupling the common features. As shown in Figure \ref{fig:pipeline}, our method consists of three main stages, including \textbf{(1)} creating shadow models, \textbf{(2)} training ownership meta-classifier, and \textbf{(3)} ownership verification with hypothesis test. The pseudocode of \texttt{Holmes} is in Appendix \ref{appendix:pseudocode}. The technical details are in the following subsections.

\begin{figure*}[!t]
    %\vspace{-1.0em}
    \centering    \includegraphics[width=0.95\textwidth]{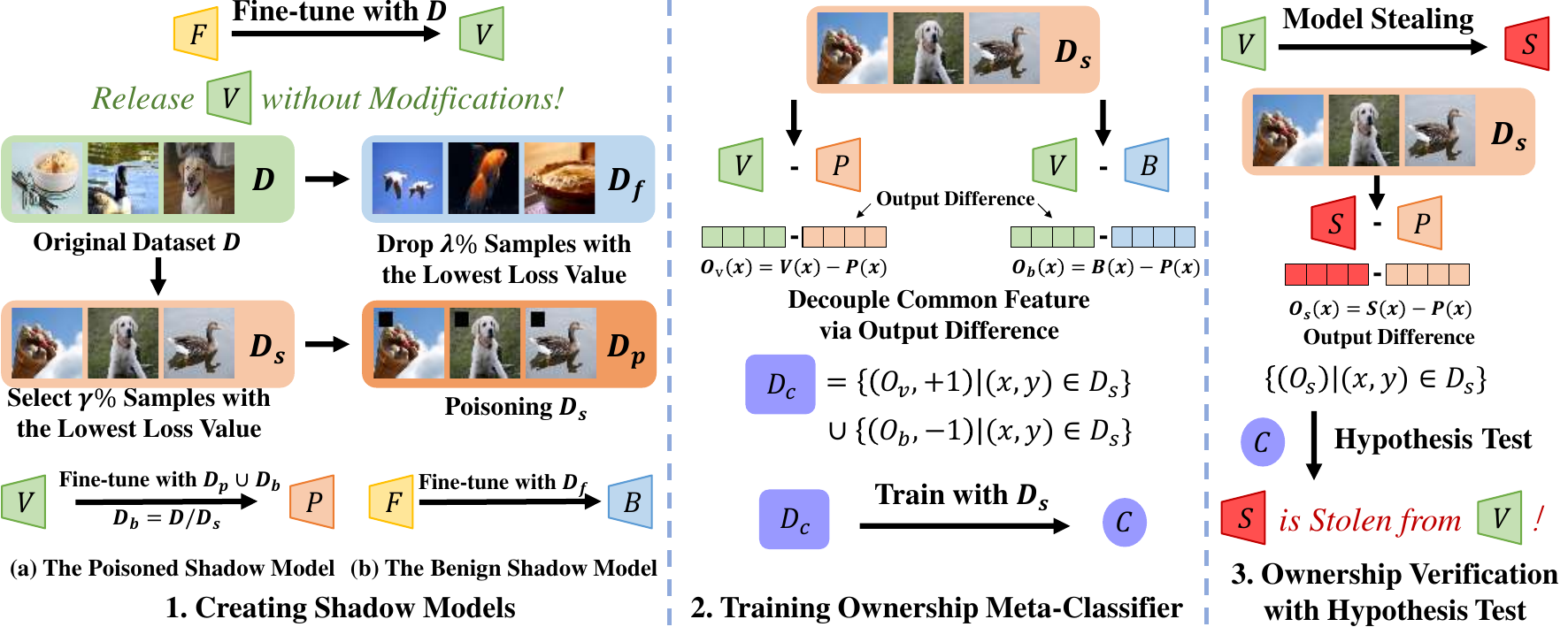}
    \vspace{0.8em}
    \caption{The main pipeline of \texttt{Holmes}. \textbf{Step 1. Creating Shadow Models:} This step involves generating two shadow models to represent dataset-specific features. (a) The poisoned shadow model disrupts dataset-specific features while preserving common features through backdoor attack. (b) The benign shadow model introduces distinct dataset-specific features by fine-tuning on a filtered dataset with similar common features. \textbf{Step 2. Training Ownership Meta-Classifier:} A meta-classifier is trained using the output differences between the shadow models and the victim model to verify ownership. \textbf{Step 3. Ownership Verification with Hypothesis Test:} The final step involves conducting model ownership verification through a hypothesis test to mitigate randomness and further enhance robustness. The verification process requires no modifications to the victim model, ensuring the harmlessness of \texttt{Holmes}.}
    \label{fig:pipeline}
    \vspace{-0.5em}
\end{figure*}

\subsection{Creating Shadow Models}
\label{sec:creating_shadow_models}

In this section, we describe how to decouple the common features and distinguish the learned dataset-specific features of the victim model $V$ via two shadow models: a poisoned shadow model $P$ and a benign shadow model $B$.

\vspace{0.3em}
\noindent \textbf{The Poisoned Shadow Model.} As illustrated in Section \ref{sec:limitation}, similar common features may lead to misjudgment. If we can erase the impact of common features, the misjudgment may be avoided. However, defining these features is challenging due to the complex learning dynamics of LVMs. Luckily, we can decouple common features indirectly.

We hereby propose to create a poisoned shadow model $P$ that retains common features learned by the victim model, while its dataset-specific features are disrupted. Thus, we can indirectly capture the dataset-specific features while decoupling the common features via the output difference between $V$ and $P$. $P$ is constructed by poisoning the most well-learned personalization samples (with the lowest loss values) with the label-inconsistent backdoor attack ($e.g.$, BadNets \cite{gu2019badnets}). We choose these samples since they correspond to the model's propensity to memorize dataset-specific patterns (instead of common features) \cite{ilyas2019adversarial}. Specifically, for the training set $\mathcal{D} = \{(\bm{x}_i, y_i) \}_{i=1}^{N}$, we select $\gamma \%$ samples with the lowest loss values (dubbed $\mathcal{D}_s$), where the loss function is cross-entropy. We adopt the label-inconsistent backdoor attack to poison $\mathcal{D}_s$, as this type of attack allows a high performance of poisoned models in predicting benign samples, thus preserving common features. Specifically, the defender first defines the target class $y_t$, trigger pattern $\bm{t}$, and a poisoned image generator $G(\cdot)$. The poisoned subset $\mathcal{D}_p$ is created as $\mathcal{D}_p = \{ (\bm{x}', y_t) |\bm{x}' = G(\bm{x}; \bm{t}), (\bm{x}, y) \in \mathcal{D}_s \}$. $P$ is created via fine-tuning $V$ with $\mathcal{D}_p \cup \mathcal{D}_b$, where $\mathcal{D}_b \triangleq \mathcal{D} \backslash \mathcal{D}_s$. Since $P$ maintains $V$'s common features while disrupting its dataset-specific features, we can capture the dataset-specific features by the output difference between $V$ and $P$ on $\mathcal{D}_s$ indirectly.

\vspace{0.3em}
\noindent \textbf{The Benign Shadow Model.} In order to further mitigate false positives from independent models with related dataset-specific features, we introduce the benign shadow model $B$ for reference. Generally, $B$ exhibits dataset-specific features that are similar but distinct from those of the victim model. Thus, the output difference between $B$ and $P$ reveals dataset-specific features that are distinct from those of $V$. To construct $B$, we fine-tune the same foundation model $F$ on a filtered dataset $\mathcal{D}_f$ which excludes $\lambda \%$ of the $V$'s training samples with the lowest loss values. By dropping these well-learned samples, $\mathcal{D}_f$ contains data that encourages $B$ to develop distinct dataset-specific features while preserving common features shared with $V$. Thus, the output difference between $B$ and $P$ highlights dataset-specific features distinct from those revealed by the output difference between $V$ and $P$.

\subsection{Training Meta-Classifier}
\label{sec:method_clf}
As we described in Section \ref{sec:creating_shadow_models}, we utilize the output difference between $V$ and $P$ on $\mathcal{D}_s$ as the representation of $V$'s dataset-specific features. Since the output difference between $V$ and $P$ is a distribution rather than a single value, we need to train an additional binary meta-classifier $C: \mathbb{R}^{|\bm{w}|} \rightarrow \{-1, +1\}$ ($\bm{w}$ represents the parameters of the meta-classifier $C$) to determine whether the suspicious model contains the knowledge of the dataset-specific features.

The key insight is to contrast two types of output differences: \textbf{(1)} the output difference between $V$ and $P$, representing unique dataset-specific features of $V$; \textbf{(2)} the output difference between $B$ and $P$, denoting features distinct from $V$'s dataset-specific features, as $B$ shares common features with $V$ but learns different dataset-specific features from $\mathcal{D}_f$. By comparing these differences, $C$ learns to identify dataset-specific features of $V$ from irrelevant features. Formally, we define the output differences of a data sample $\bm{x}$ as:
\begin{equation}
\begin{aligned}
    \bm{O}_v(\bm{x}) =& V(\bm{x})-P(\bm{x}),\\
    \bm{O}_b(\bm{x}) =& B(\bm{x})-P(\bm{x}).
\end{aligned}
\end{equation}

The training set $\mathcal{D}_c$ of meta-classifier $C$ is constructed by labeling $\bm{O}_v(\bm{x})$ as positive samples (+1) and $\bm{O}_b(\bm{x})$ as negative samples (-1):
\begin{equation}
\begin{aligned}
   \mathcal{D}_c =&  \left\{\left(\bm{O}_v(\bm{x}), +1\right)| (\bm{x}, y) \in \mathcal{D}_s \right\} \cup \\
    & \left\{\left(\bm{O}_b(\bm{x}), -1\right)| (\bm{x}, y) \in \mathcal{D}_s \right\}.
\end{aligned}
\end{equation}

Finally, the meta-classifier $C$ is trained by minimizing the following objective: 
\begin{equation}
    \arg \min_{\bm{w}} \sum_{(\bm{O}, l) \in \mathcal{D}_c} \mathcal{L}(C_{\bm{w}}(\bm{O}), l),
\end{equation}
where $\mathcal{L}(\cdot)$ is the loss function ($i.e.$, cross-entropy), $\bm{w}$ represents the parameters of the meta-classifier $C$, and $l \in \{+1, -1\}$ denotes the label indicating whether the input corresponds to the dataset-specific features ($i.e.$, `$l=+1$') learned by victim model or unrelated features ($i.e.$, `$l=-1$').

\subsection{Ownership Verification with Hypothesis Test}
\label{sec:hypothesis_test}

In the verification process, given a suspicious model $S$, the defender can examine it via the meta-classifier output $C(\bm{O}_s(\bm{x}))$ for samples $(\bm{x}, y) \in \mathcal{D}_s$, where $\bm{O}_s(\bm{x}) = S(\bm{x})-P(\bm{x})$. If $C(\bm{O}_s(\bm{x})) = +1$, $S$ is considered stolen from the victim. However, this decision is susceptible to the randomness of the selected $\bm{x}$. To mitigate this issue, we adopt a hypothesis test-based method inspired by existing works \cite{maini2021dataset, li2022defending, li2025move}, as described below:

\begin{definition} Let $\bm{X}$ be a random variable representing samples from $\mathcal{D}_s$. Define $\mu_{S}$ and $\mu_{B}$ as the posterior probabilities of the events $C(\bm{O}_s(\bm{x})) = +1$ and $C(\bm{O}_s(\bm{x})) = -1$, respectively. Given the null hypothesis $H_0: \mu_{B} + \tau = \mu_{S}$ ($H_1: \mu_{B} + \tau < \mu_{S}$) where $\tau \in [0, 1]$, we claim that the suspicious model $S$ is stolen from the victim (with $\tau$-certainty) if and only if $H_0$ is rejected.
\end{definition}

In practice, we randomly sample $m$ samples from $\mathcal{D}_s$ to conduct the pair-wise T-test \cite{hogg2005introduction} and calculate its corresponding p-value. When the p-value is smaller than the pre-defined significance level $\alpha$, $H_0$ is rejected. Besides, we also calculate the \emph{confidence score} $\Delta \mu = \mu_{S} - \mu_{B}$ to represent the verification confidence. The larger the $\Delta \mu$, the more confident the verification.

\subsection{Robustness Analysis of \texttt{Holmes}'s Ownership Verification}
\label{sec:robustness_analysis}

Since perfect accuracy (100\%) is usually unachievable for the meta-classifier $C$, it raises concerns about the reliability of our ownership verification under classifier imperfections. In this section, we demonstrate that our verification framework remains statistically robust when $C$ exhibits realistic error rates. The success condition of verification is formalized as follows.

\begin{theorem}\label{thm1}
Let $\bm{X}$ be a random variable representing samples from $\mathcal{D}_s$. Assume that $\mu_{B} \triangleq \mathbb{P}(C(\bm{O}_s(\bm{X})) = -1) < \beta$. We claim that the verification process can reject the null hypothesis $H_0$ at the significance level $\alpha$ if the identification success rate $R$ of $C$ satisfies that
\begin{equation}
    R > \frac{2(m-1)(\beta+\tau)+t_{1-\alpha}^2 + \sqrt{\Delta} }{2(m-1+t_{1-\alpha}^2)},
\end{equation}
where $\Delta = t_{1-\alpha}^4 + 4t_{1-\alpha}^2(m-1)(\beta+\tau)(1-\beta - \tau)$, $t_{1-\alpha}$ is the $(1-\alpha)$-quantile of $t$-distribution with $(m-1)$ degrees of freedom, and m is the sample size of $\bm{X}$.
\end{theorem}

In general, Theorem \ref{thm1} indicates two critical properties: \textbf{(1)} Our hypothesis test-based ownership verification can succeed if the identification success rate of $C$ is higher than a threshold (which is not necessarily 100\%). \textbf{(2)} The defender can claim the model ownership with limited queries to the suspicious model $S$ if the identification success rate is high enough. The detailed proof of Theorem \ref{thm1} is included in Appendix \ref{appendix:proof_theorem_1}.

%\subsection{Task Adaptation}

\section{Experiments}

\label{sec:main_experiments}
\begin{figure*}[!t]
%\vspace{-1.5em}
\centering
\subfloat[]{\includegraphics[width=0.15\textwidth]{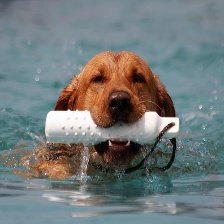}%
\label{fig:original}}
\hfil
\subfloat[]{\includegraphics[width=0.15\textwidth]{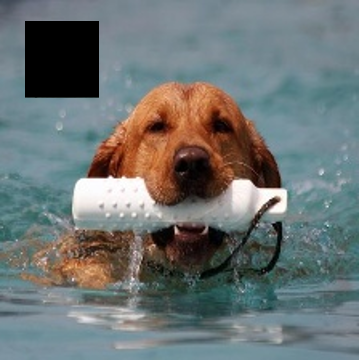}%
\label{fig:badnets}}
\hfil
\subfloat[]{\includegraphics[width=0.15\textwidth]{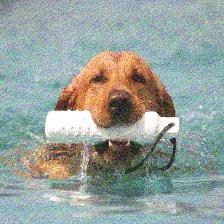}%
\label{fig:gm}}
\hfil
\subfloat[]{\includegraphics[width=0.15\textwidth]{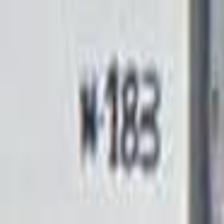}%
\label{fig:ewe}}
\hfil
\subfloat[]{\includegraphics[width=0.15\textwidth]{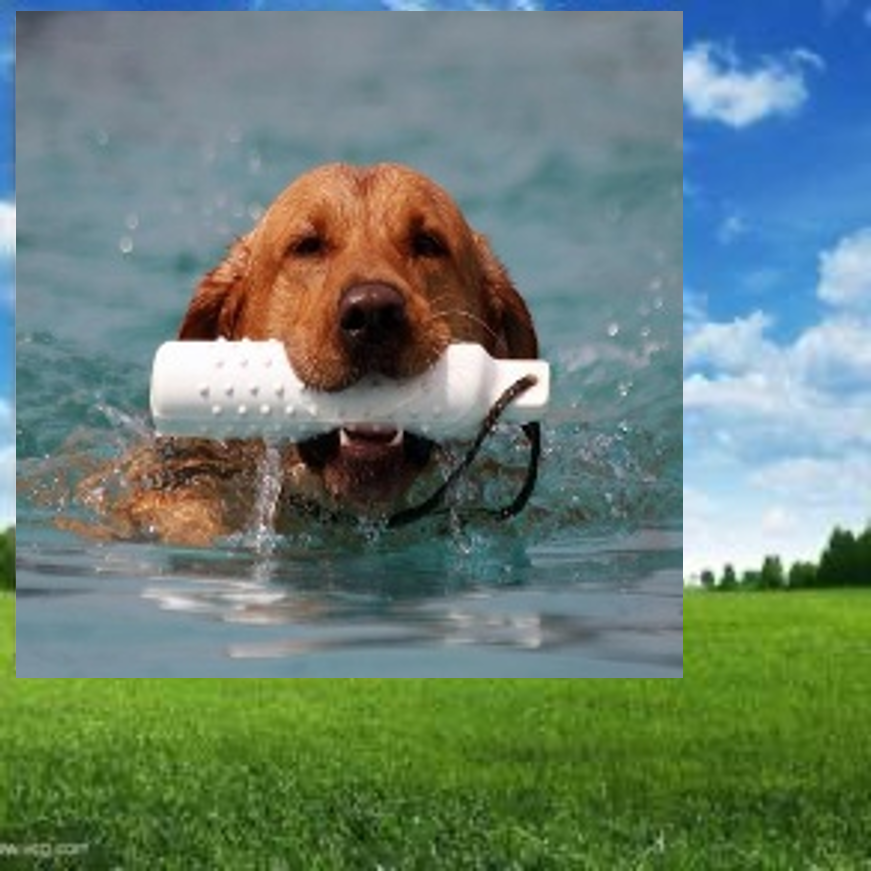}%
\label{fig:ptynet}}
\hfil
\\
\vspace{-0.8em}
\subfloat[]{\includegraphics[width=0.15\textwidth]{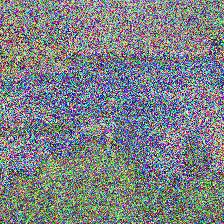}%
\label{fig:uae}}
\hfil
\subfloat[]{\includegraphics[width=0.15\textwidth]{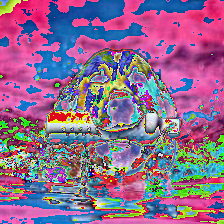}%
\label{fig:uaps}}
\hfil
\subfloat[]{\includegraphics[width=0.15\textwidth]{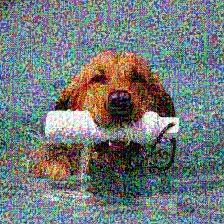}%
\label{fig:metafinger}}
\hfil
\subfloat[]{\includegraphics[width=0.15\textwidth]{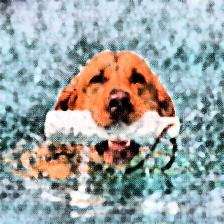}%
\label{fig:move}}
\hfil
\subfloat[]{\includegraphics[width=0.15\textwidth]{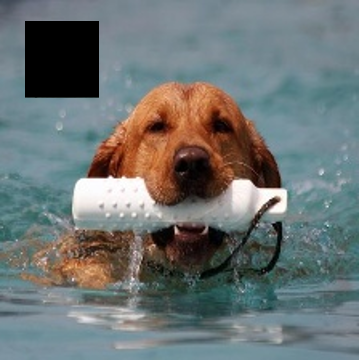}%
\label{fig:Holmes}}
\caption{Images involved in different defenses. \textbf{(a)} original image; \textbf{(b)} poisoned image in BadNets; \textbf{(c)} poisoned image in GM; \textbf{(d)} out-of-distribution watermark image from SVHN dataset in EWE; \textbf{(e)} poisoned image in PTYNet; \textbf{(f)} noised image in UAE; \textbf{(g)} noised image in UAPs; \textbf{(h)} noised image in Metafinger; \textbf{(i)} transformed image in MOVE; \textbf{(j)} poisoned image in \texttt{Holmes} for creating the poisoned shadow model.}
\label{fig:example_figure}
\vspace{-0.8em}
\end{figure*}

\subsection{Main Settings}
\label{sec:main_setting}

\vspace{0.3em}
\noindent \textbf{Dataset Selection.} We evaluate the effectiveness of \texttt{Holmes} using the CIFAR-10 \cite{krizhevsky2009learning} and ImageNet \cite{deng2009imagenet} datasets. CIFAR-10 consists of 60,000 images across 10 categories, with 50,000 samples allocated for training and 10,000 for testing. For ImageNet, we use a subset of 200 randomly selected classes to balance computational efficiency and representativeness. Each class in this subset includes 2,500 randomly selected samples for the training set and 50 samples for the testing set.

\vspace{0.3em}
\noindent \textbf{Settings of Personalization.} We adopt the CLIP \cite{radford2021learning} model (with a ViT-B/32 Transformer architecture \cite{vaswani2017attention} as the image encoder) as the victim model for experiments on the CIFAR-10 and ImageNet datasets. Specifically, we use the pre-trained CLIP model from the public repository\footnote{\url{https://huggingface.co/openai/clip-vit-base-patch32}} as the foundation model. For personalization, the victim models are obtained by fine-tuning the pre-trained foundation model on CIFAR-10 and ImageNet datasets for 20 epochs, respectively. The learning rate is set to $1 \times 10^{-5}$. To improve fine-tuning effectiveness, the learning rate for the linear layers ($i.e.$, the final layers of CLIP models) is increased by a factor of 10, reaching $1 \times 10^{-4}$.

\vspace{0.3em}
\noindent \textbf{Settings of Model Stealing.} Following the settings in \cite{maini2021dataset, li2022defending, li2025move}, we perform the model stealing attacks described in Section \ref{sec:model-stealing} to assess the effectiveness of the proposed method.

Before detailing each attack, we first clarify the architectures and datasets used for the stolen models across different scenarios. For \emph{Direct-copy} and \emph{Fine-tuning} attacks, the stolen model adopts the same architecture as the victim model. For all distillation-based attacks, including distillation with access to the training dataset (\emph{Data Distill}), data-free distillation (\emph{Data-free Distill}), \emph{Hard Distill}, and \emph{Soft Distill}, we employ a ResNet-18 network as the stolen model.

Regarding the datasets used in model stealing, we construct substitute datasets for \emph{Fine-tuning}, \emph{Hard Distill}, and \emph{Soft Distill} attacks, following \cite{maini2021dataset, li2022defending, li2025move}. For CIFAR-10 experiments, we randomly sample 500,000 images from the TinyImages dataset \cite{birhane2021large} as the substitute dataset. For ImageNet experiments, we select 2,500 samples from the same 200 classes used to fine-tune the victim model, ensuring that all selected samples are disjoint from those used during personalization.

In particular, \textbf{(1)} \emph{Direct-copy}: The adversary directly copies the victim model. \textbf{(2)} \emph{Fine-tuning}: The adversary obtains stolen models by fine-tuning victim models for 5 epochs on substitute datasets. \textbf{(3)} \emph{Data Distill}: We implement the CTKD method \cite{li2023curriculum} using open-source code\footnote{\url{https://github.com/zhengli97/ctkd}}. The stolen model is trained for 100 epochs on the training set of $V$ using the SGD optimizer with an initial learning rate of 0.1, momentum of 0.9, and a weight decay of $1 \times 10^{-4}$. \emph{(4) Data-free Distill}: We adopt NAYER \cite{tran2024nayer} as the attack method, using its open-source code\footnote{\url{https://github.com/tmtuan1307/NAYER}}. For both datasets, the stolen model is trained for 300 epochs, with other settings following the original work. \emph{(5) Hard Distill}: The stolen model is trained for 200 epochs on the substitute datasets using hard labels generated by the victim model. We adopt the SGD optimizer with an initial learning rate of 0.1, momentum of 0.9, and a weight decay of $1 \times 10^{-4}$. \emph{(6) Soft Distill}: The stolen model is trained using logit-based knowledge distillation \cite{sun2024logit}, implemented with the open-source code\footnote{\url{https://github.com/sunshangquan/logit-standardization-KD}}. Training is performed for 100 epochs using the SGD optimizer with an initial learning rate of 0.1, momentum of 0.9, and a weight decay of $1 \times 10^{-4}$. Besides, we conduct experiments of examining a suspicious model which is not stolen from the victim (dubbed `Independent') for reference. We utilize a pre-trained CLIP model with a ViT-B/16 Transformer architecture as the image encoder, sourced from the public repository\footnote{\url{https://huggingface.co/openai/clip-vit-base-patch16}}. This foundation model is fine-tuned on substitute datasets to create independent models for CIFAR-10 and ImageNet experiments, respectively. The fine-tuning settings for these independent models are consistent with those used for the victim models.

\vspace{0.3em}
\noindent \textbf{Settings of Defenses.} We compare \texttt{Holmes} with model watermarking and model fingerprinting methods. For model watermarking baselines, we examine methods including BadNets watermarking \cite{gu2019badnets} implemented via BackdoorBox \cite{li2022backdoorbox}, gradient matching (GM) \cite{geiping2021witches}, entangled watermarks (EWE) \cite{jia2021entangled}, PTYNet \cite{wang2023free}, UAE \cite{zhu2024reliable}, and MOVE \cite{li2025move}. All watermarking methods employ a consistent 10\% modification rate on training samples, with implementation details following their official open-source repositories. Figure \ref{fig:example_figure} illustrates representative examples of modified images across different defenses, including BadNets' black square trigger, poisoned image in GM, EWE's out-of-distribution samples from SVHN dataset \cite{netzer2011reading}, PTYNet's semantic background triggers, and the oil painting style image of MOVE. Specifically, MOVE adopts different approaches for white-box and black-box settings. Since this paper focuses on verification in black-box settings, all experiments are conducted under black-box settings.

For model fingerprinting baselines, we compare \texttt{Holmes} against dataset inference (DI) \cite{maini2021dataset}, universal adversarial perturbations (UAPs) \cite{peng2022fingerprinting}, and Metafinger \cite{yang2022metafinger}. DI adopts different approaches for white-box and black-box settings. Following the black-box verification focus of this work, DI is evaluated under its black-box configurations.

For our method, we construct the poisoned model $P$ by fine-tuning $V$ on $\mathcal{D}_p \cup \mathcal{D}_b$ ($\mathcal{D}_b \triangleq \mathcal{D} \backslash \mathcal{D}_s$) for 5 epochs. For both CIFAR-10 and ImageNet, we set the modification rate $\gamma \%=10\%$, resulting in $|\mathcal{D}_s|=5,000$ (CIFAR-10) and $|\mathcal{D}_s|=500$ (ImageNet). The poisoned trigger is a black square in the top-left corner, as shown in Figure \ref{fig:Holmes}. The target label is set to ``automobile'' for CIFAR-10 and ``white shark'' for ImageNet. The benign shadow model $B$ is obtained by fine-tuning the foundation model $F$ on $\mathcal{D}_f$ for 20 epochs with a dropping rate $\lambda \%=50\%$ for both datasets. Both $P$ and $B$ use a base learning rate of $1\times10^{-5}$. For $B$, the learning rate for the final linear layers increases to $1\times10^{-4}$ to enhance fine-tuning effectiveness. The meta-classifier $C$ is trained for 300 epochs with a learning rate of 0.01. The hypothesis test is conducted at significance level $\alpha=0.01$.

\begin{table*}[!ht]
\centering
\vspace{-2em}
  \caption{Results of defenses on the CIFAR-10 dataset. We mark failed verification cases ($i.e.$, p-value $>0.01$ for stolen models and p-value $<0.01$ for independent models) in red.}
\scalebox{0.655}{
\begin{tabular}{c|l|cc|cc|cc|cc|cc|cc||cc}
\toprule
\multicolumn{2}{c|}{Model Stealing$\rightarrow$ }& \multicolumn{2}{c|}{Direct-copy}  &  \multicolumn{2}{c|}{Fine-tuning
}&\multicolumn{2}{c|}{Data Distill}& \multicolumn{2}{c|}{Data-free Distill} & \multicolumn{2}{c|}{Hard Distill}   & \multicolumn{2}{c||}{Soft Distill}   & \multicolumn{2}{c}{Independent}  \\
\hline
\multicolumn{2}{c|}{Defense$\downarrow$}&$\Delta \mu$& p-value& $\Delta \mu$ & p-value &$\Delta \mu$& p-value&$\Delta \mu$ & p-value&$\Delta \mu$ & p-value&$\Delta \mu$ & p-value&$\Delta \mu$ & p-value\\
\hline
\multirow{6}{*}{Watermark} & BadNets & \textbf{0.84}& $\mathbf{10^{-60}}$ &  \textbf{0.42}& $\mathbf{10^{-21}}$ 
&\red{$10^{-3}$}&\red{0.05}& \red{-0.09}& \red{0.10}& \red{$10^{-3}$}&\red{0.16}& \red{$10^{-3}$}&\red{0.08}& $10^{-3}$  &0.08\\
       & GM        & 0.85& $10^{-57}$ &  $10^{-3}$  & $10^{-3}$  
&\red{0.01}& \red{0.07}&$10^{-3}$	&0.01& \red{$10^{-3}$}& \red{0.17}& \red{$10^{-3}$}&\red{0.17}& $10^{-3}$  & 0.16\\
       & EWE & 0.02& $10^{-26}$ &  \red{-$10^{-3}$} & \red{0.99}
&\red{-$10^{-6}$} & \red{0.99}& \red{-0.04}& \red{0.71}& \red{-$10^{-3}$} & \red{0.99}& \red{-$10^{-3}$} & \red{0.99}& $\mathbf{-10^{-3}}$ & \textbf{0.99}\\
       & PTYNet     & 0.22& $10^{-4}$  &  \red{-0.08}& \red{0.99}&\red{-0.07} & \red{0.17}& \red{-0.14}&\red{0.93}& \red{-0.04}& \red{0.99}& \red{-0.03}& \red{0.98}& \textbf{-0.02}& \textbf{0.99}\\
       & UAE        & \red{$10^{-3}$} & \red{0.17}&  \red{$10^{-3}$}  & \red{0.15}
&\red{-$10^{-3}$} & \red{0.24}& \red{-0.14}&\red{0.95}& \red{-$10^{-3}$} & \red{0.23}& \red{-$10^{-4}$} & \red{0.15}& -$10^{-3}$ & 0.11\\
& MOVE& $10^{-4}$ & $10^{-5}$  &  $10^{-4}$& $10^{-5}$&\red{$10^{-4}$}  & \red{0.30}&\red{$10^{-3}$}&	\red{0.10}&\red{$10^{-4}$}&\red{0.17}& $10^{-3}$  & 0.01& \textbf{0.00}  & \textbf{0.99}\\
\hline
\multirow{4}{*}{Fingerprint}  &UAPs&0.18&$10^{-4}$& \red{$10^{-3}$}& \red{0.08}&\red{0.08}
&\red{0.18}&\red{-0.16}&\red{0.97}&\red{-0.07}&\red{0.19}&\red{-0.09}&\red{0.25}&-$10^{-3}$&0.22 \\
       & Metafinger  & 0.36 & $10^{-3}$&  \red{$10^{-3}$}  & \red{0.27}
&\red{-0.01}& \red{0.76} &\red{-0.15}	&\red{0.83}& \red{0.01}& \red{0.02}& \red{-0.01}& \red{0.67}& -$10^{-4}$ & 0.49\\
       & DI & 0.96& $10^{-15}$ &  0.96& $10^{-15}$ 
&\textbf{0.91}& $\mathbf{10^{-14}}$&0.10	&$10^{-3}$ & \red{0.02}& \red{0.28}& \red{0.03}& \red{0.26}& \red{0.72}& \red{$10^{-9}$} \\
\cmidrule(lr){2-16}
& \texttt{Holmes} & 0.18& $10^{-14}$ &  0.12& $10^{-9}$&0.18& $10^{-12}$ &\textbf{0.13}&$\mathbf{10^{-9}}$& \textbf{0.05}& $\mathbf{10^{-4}}$ & \textbf{0.14}& $\mathbf{10^{-8}}$ & -0.03& 0.20\\
\bottomrule
\end{tabular}
}
\label{table:main-cifar10} 
\vspace{-1.5em}
\end{table*}

\begin{table*}[!ht]
\centering
\caption{Results of defenses on the ImageNet dataset. We mark failed verification cases ($i.e.$, p-value $>0.01$ for stolen models and p-value $<0.01$ for independent models) in red.}
%\vspace{-0.6em}
\scalebox{0.645}{
\begin{tabular}{c|l|cc|cc|cc|cc|cc|cc||cc}
\toprule
\multicolumn{2}{c|}{Model Stealing $\rightarrow$ }& \multicolumn{2}{c|}{Direct-copy}  &  \multicolumn{2}{c|}{Fine-tuning}&\multicolumn{2}{c|}{Data Distill}& \multicolumn{2}{c|}{Data-free Distill} & \multicolumn{2}{c|}{Hard Distill}   & \multicolumn{2}{c||}{Soft Distill}   & \multicolumn{2}{c}{Independent}  \\
\hline
\multicolumn{2}{c|}{Defense $\downarrow$}&$\Delta \mu$& p-value& $\Delta \mu$ & p-value
&$\Delta \mu$& p-value&$\Delta \mu$ & p-value&$\Delta \mu$ & p-value&$\Delta \mu$ & p-value&$\Delta \mu$ & p-value\\
\hline
\multirow{6}{*}{Watermark} & BadNets & 0.04 & $10^{-5}$ &  \red{$10^{-4}$} & \red{0.38} 
&\red{0.01} & \red{0.05} & \red{-0.01} & \red{0.75} & \red{$10^{-3}$} & \red{0.18} & \red{$10^{-3}$} & \red{0.17} & $10^{-5}$ & 0.31 \\
 & GM & 0.85 & $10^{-31}$ &  \red{$10^{-3}$} & \red{0.33} 
&\red{$10^{-3}$} & \red{0.18} & \red{-0.01} & \red{0.66} & \red{0.01} & \red{0.24} & \red{$10^{-3}$} & \red{0.27} & $10^{-3}$ & 0.21 \\
 & EWE & \red{$-10^{-5}$} & \red{0.99} &  \red{-0.02} & \red{0.99} 
&\red{$-10^{-3}$} & \red{0.94} & \red{-0.07} & \red{0.99} & \red{$-10^{-3}$} & \red{0.98} & \red{-0.01} & \red{0.98} & \textbf{-0.06} & \textbf{0.99} \\
 & PTYNet &0.11 & $10^{-3}$ &  \red{-0.02} & \red{0.08} 
&\red{-0.02} & \red{0.84} & \red{-0.03} & \red{0.96} & \red{-0.03} & \red{0.75} & \red{-0.01} & \red{0.03} & -0.01 & 0.52 \\
 & UAE  & 0.35 & $10^{-3}$ &  \red{$10^{-4}$} & \red{0.26} &\red{$10^{-4}$} & \red{0.22} & \red{$10^{-5}$} & \red{0.29} & \red{$10^{-7}$} & \red{0.21} & \red{$-10^{-7}$} & \red{0.33} & $10^{-5}$ & 0.23 \\
 & MOVE & $10^{-4}$ & $10^{-4}$ &  $10^{-4}$ & $10^{-5}$ 
&\red{$10^{-4}$} & \red{0.05} & 0.11 & $10^{-3}$ & \red{$-10^{-4}$} & \red{0.44} & \red{$-10^{-4}$} & \red{0.56} & \textbf{0.00} & \textbf{0.99} \\
 \hline
\multirow{4}{*}{Fingerprint} & UAPs & 0.52 & $10^{-7}$ &  0.01 & $10^{-3}$ 
&\red{$10^{-4}$} & \red{0.27} &\red{ -0.01} & \red{0.13} & \red{-0.10} & \red{0.33} & \red{-0.12} & \red{0.24} & -0.01 & 0.21 \\
 & Metafinger & $10^{-3}$ & $10^{-5}$ &  \red{$10^{-5}$} & \red{0.35} 
&\red{$10^{-3}$} & \red{0.31} & \red{-0.12} & \red{0.30} & \red{$10^{-4}$} & \red{0.36} & \red{$10^{-6}$} & \red{0.07} &$10^{-5}$ & 0.38 \\
 & DI & 0.54 & $10^{-8}$ &  0.54 & $10^{-9}$ 
&0.51 & $10^{-8}$ & 0.15 & $10^{-3}$ & 0.51 & $10^{-8}$ & 0.39 & $10^{-6}$ & \red{0.54} & \red{$10^{-8}$} \\
 \cmidrule(lr){2-16}
 & \texttt{Holmes} & \textbf{0.37} & $\mathbf{10^{-36}}$ &  \textbf{0.17} & $\mathbf{10^{-10}}$ &\textbf{0.27} & $\mathbf{10^{-20}}$ & \textbf{0.10} & $\mathbf{10^{-4}}$ & \textbf{0.37} & $\mathbf{10^{-32}}$ & \textbf{0.22} & $\mathbf{10^{-11}}$ & \textbf{-0.61} & \textbf{0.99}\\
 \bottomrule
\end{tabular}
}
\label{table:main-imagenet}
%\vspace{-0.5em}
\end{table*}

\vspace{0.3em}
\noindent \textbf{Evaluation Metric.} We use the confidence score $\Delta \mu$ and p-value as the metric for our evaluation. Both of them are calculated based on the hypothesis test with 100 sampled images, $i.e.$, $m=100$. For cases involving stolen models, a lower p-value and a higher $\Delta \mu$ indicate stronger verification performance. Conversely, for independent models, a higher p-value and a lower $\Delta \mu$ suggest better resistance to misclassification. The best results among all defenses are highlighted in bold, while failed verifications ($i.e.$, p-value $>\alpha$ for stolen models and p-value $<\alpha$ for independent models, where $\alpha=0.01$) are marked in red.

We note that some baseline methods ($e.g.$, model watermarking with BadNets) do not use hypothesis test for model ownership verification. These methods typically determine model ownership by assessing whether the probability assigned by the suspicious model to a watermarked sample for a specific target label is significantly higher than that assigned to its original version. For fair comparison, we also adopt hypothesis test for these methods, as follows: 

\begin{definition} Let $\bm{X}'$ and $\bm{X}$ denote the variable of the watermarked sample and its benign version, respectively. Define $\mu_{S}$ and $\mu_{B}$ as $\mu_{S} \triangleq S(\bm{X}')_{y_t}$ and $\mu_{B} \triangleq S(\bm{X})_{y_t}$, where $y_t$ is the target label. The $\Delta \mu$ is calculated as $\Delta \mu = \mu_{S} - \mu_{B}$. Given the null hypothesis $H_0: \mu_{B} + \tau = \mu_{S}$ ($H_1: \mu_{B} + \tau < \mu_{S}$) where $\tau \in [0, 1]$, we claim that the suspicious model $S$ is stolen from the victim (with $\tau$-certainty) if and only if $H_0$ is rejected.
\end{definition}

\subsection{Main Results}

As shown in Tables \ref{table:main-cifar10}-\ref{table:main-imagenet}, \texttt{Holmes} is the only method that achieves reliable verification across all model stealing scenarios, outperforming baselines in most cases. For example, in the Soft Distill attack on the CIFAR-10 dataset, our method achieves a p-value of $10^{-8}$, significantly outperforming MOVE, the only baseline capable of successful verification in this scenario, which achieves a p-value of 0.01. In the cases of Direct-copy, Data Distill, and Fine-tuning, methods such as BadNets achieve competitive results. Nevertheless, our method can still easily make correct predictions in these cases. Besides, on the ImageNet dataset, our method demonstrates superior performance across all scenarios, including both model stealing cases and the independent model case. For example, in the Data Distill attack scenario, our method achieves a p-value of $10^{-20}$, outperforming DI's $10^{-8}$, while all other baseline methods fail to provide model ownership verification. Additionally, in the independent model scenario for both CIFAR-10 and ImageNet datasets, our method consistently avoids misjudgments, while DI misclassifies the benign independent models as stolen models. This indicates that our method effectively decouples common features while emphasizing dataset-specific features, ensuring reliable model ownership verification across diverse datasets and model stealing attack methods.

\subsection{Effects of Hyper-parameters}
In this section, we discuss the effects of hyper-parameters involved in our method.

\subsubsection{Effects of Modification Rate}

\begin{table*}[!t]
\centering
\vspace{-1.5em}
\caption{Effects of modification rate of the poisoned shadow model.}
%\vspace{-0.6em}
\scalebox{0.69}{
\begin{tabular}{c|cc|cc|cc|cc|cc|cc||cc}
\toprule
Model Stealing $\rightarrow$& \multicolumn{2}{c|}{Direct-copy} &  \multicolumn{2}{c|}{Fine-tuning}&\multicolumn{2}{c|}{Data Distill} & \multicolumn{2}{c|}{Data-free Distill}  & \multicolumn{2}{c|}{Hard Distill} & \multicolumn{2}{c||}{Soft Distill} & \multicolumn{2}{c}{Independent} \\
\hline
Modification Rate $\downarrow$& $\Delta \mu$& p-value &  $\Delta \mu$& p-value 
&$\Delta \mu$& p-value & $\Delta \mu$& p-value & $\Delta \mu$& p-value & $\Delta \mu$& p-value & $\Delta \mu$& p-value \\
\hline
0\% & {-0.03} & {0.40} &  {-0.21} & {0.99} 
&{-0.09} & 
{0.12} & {-0.01} & {0.12} & {$-10^{-3}$} & {0.33} & {-0.30} & {0.51} & {0.07} & {$10^{-4}$} \\
\hline
5\% & 0.17 & $10^{-11}$ &  0.02 & 0.01 
&0.16 & $10^{-8}$ & 0.23 & $10^{-9}$ & 0.05 & $10^{-4}$ & 0.27 & $10^{-16}$ & -0.04 & 0.30 \\
10\% & 0.18 & $10^{-14}$ &  0.12 & $10^{-9}$ 
&0.18 & $10^{-12}$ & 0.13 & $10^{-9}$ & 0.05 & $10^{-4}$ & 0.14 & $10^{-8}$ & -0.03 & 0.20 \\
15\% & 0.19 & $10^{-12}$ &  $10^{-3}$ & $10^{-3}$ 
&0.15 & $10^{-8}$ & 0.23 & $10^{-16}$ & 0.03 & $10^{-3}$ & 0.22 & $10^{-21}$ & -0.20 & 0.29 \\
20\% & 0.19 & $10^{-15}$ &  0.14 & $10^{-5}$ &0.07 & $10^{-10}$ & 0.12 & $10^{-3}$ & 0.02 & $10^{-3}$ & 0.09 & $10^{-7}$ & -0.01 & 0.04\\
\bottomrule
\end{tabular}
}
\label{table:poison_rate}
\vspace{-1.5em}
\end{table*}

\begin{table*}[!t]
\centering
\caption{Effects of drop rate of the benign shadow model.}
\scalebox{0.715}{
\begin{tabular}{c|cc|cc|cc|cc|cc|cc||cc}
\toprule
Model Stealing $\rightarrow$ & \multicolumn{2}{c|}{Direct-copy} &  \multicolumn{2}{c|}{Fine-tuning
}&\multicolumn{2}{c|}{Data Distill} & \multicolumn{2}{c|}{Data-free Distill}  & \multicolumn{2}{c|}{Hard Distill} & \multicolumn{2}{c||}{Soft Distill} & \multicolumn{2}{c}{Independent} \\
\hline
Drop Rate $\downarrow$& $\Delta \mu$& p-value &  $\Delta \mu$& p-value 
&$\Delta \mu$& p-value & $\Delta \mu$& p-value & $\Delta \mu$& p-value & $\Delta \mu$& p-value & $\Delta \mu$& p-value \\
\hline
0\% & 0.29 & $10^{-12}$ &  {-0.07} & {0.38} 
&{$10^{-5}$} & {0.59} & {-0.01} & {0.63} & -0.02& {0.02} & {-0.04} & {0.22} & {0.13} & {$10^{-7}$} \\
\hline
10\% & 0.20 & $10^{-14}$ &  0.01 & $10^{-9}$ 
&0.17 & $10^{-9}$ & 0.17 & $10^{-13}$ & 0.04 & $10^{-3}$ & 0.08 & $10^{-5}$ & -0.03 & 0.15 \\
30\% & 0.15 & $10^{-14}$ &  0.12 & $10^{-9}$ 
&0.26 & $10^{-11}$ & 0.25 & $10^{-14}$ & 0.03 & $10^{-4}$ & 0.15 & $10^{-8}$ & -0.07 & 0.19 \\
50\% & 0.18 & $10^{-14}$ &  0.12 & $10^{-9}$ 
&0.18 & $10^{-12}$ & 0.13 & $10^{-9}$ & 0.05 & $10^{-4}$ & 0.14 & $10^{-8}$ & -0.03 & 0.20 \\
70\% & 0.21 & $10^{-26}$  &   0.16 & $10^{-7}$  &0.30 & $10^{-18}$  & 0.25 & $10^{-25}$  & 0.20 & $10^{-11}$ &  0.25 &  $10^{-15}$ &  -0.03 & 0.19 \\
\bottomrule
\end{tabular}
}
\label{table:drop_rate}
%\vspace{-1.5em}
\end{table*}

\begin{table*}[!t]
\centering
\vspace{-2em}
\caption{Effects of number of verification samples.}
\scalebox{0.725}{
\begin{tabular}{c|cc|cc|cc|cc|cc|cc||cc}
\toprule
Model Stealing$\rightarrow$ & \multicolumn{2}{c|}{Direct-copy} &  \multicolumn{2}{c|}{Fine-tuning}&\multicolumn{2}{c|}{Data Distill} & \multicolumn{2}{c|}{Data-free Distill}  & \multicolumn{2}{c|}{Hard Distill} & \multicolumn{2}{c||}{Soft Distill} & \multicolumn{2}{c}{Independent} \\
\hline
$m$$\downarrow$& $\Delta \mu$ & p-value &  $\Delta \mu$& p-value 
&$\Delta \mu$& p-value & $\Delta \mu$& p-value & $\Delta \mu$& p-value & $\Delta \mu$& p-value & $\Delta \mu$& p-value \\
\hline
$m=50$ & 0.17 & $10^{-7}$ &  0.11 & $10^{-9}$ 
&0.18 & $10^{-7}$ & 0.13 & $10^{-4}$ & 0.05 & $10^{-3}$ & 0.12 & $10^{-6}$ & -0.02 & 0.14 \\
$m=100$ & 0.18 &$10^{-14}$ &  0.12 & $10^{-9}$ 
&0.18 & $10^{-12}$ & 0.13 & $10^{-9}$ & 0.05 & $10^{-4}$ & 0.14 & $10^{-8}$ & -0.03 & 0.20 \\
$m=150$ & 0.18 & $10^{-14}$ &  0.12 & $10^{-9}$ 
&0.18 & $10^{-15}$ & 0.14 & $10^{-9}$ & 0.05 & $10^{-4}$ & 0.14 & $10^{-10}$ & -0.03 & 0.28 \\
$m=200 $& 0.19 & $10^{-19}$ &  0.12 & $10^{-9}$ &0.17 & $10^{-15}$ & 0.14 & $10^{-10}$ & 0.05 & $10^{-4}$ & 0.14 & $10^{-10}$ & -0.03 & 0.36 \\
\bottomrule
\end{tabular}
}
\label{table:sample_number}

\end{table*}

In general, a higher modification rate $\gamma \%$ leads to greater disruption of dataset-specific features and less preservation of common features in model $P$. As shown in Table \ref{table:poison_rate}, when the modification rate is 0\% ($i.e.$, no poisoning is applied), the method fails to distinguish between stolen and benign models, leading to misjudgments. This is because the output difference between the victim model and the unpoisoned shadow model (which is essentially similar to the victim model) fails to isolate the dataset-specific features, causing the meta-classifier to fail to identify dataset-specific features. Conversely, when $\gamma\% \in \{5\%,10\%,15\%,20\%\}$, the method achieves accurate verification across all cases. Specifically, for the Direct-copy attack, $\gamma \%$ of 5\% yields a p-value as low as $10^{-11}$, indicating effective verification. As $\gamma \%$ increases from 5\% to 20\%, the p-values exhibit a slight downward trend ($e.g.$, from $10^{-11}$ to $10^{-15}$ for Direct-copy attacks), reflecting reliable disruption of dataset-specific features.

\subsubsection{Effects of Drop Rate}
We drop $\lambda \%$ training samples with the lowest loss values from $V$'s training dataset to ensure that $B$ learns different dataset-specific features while preserving similar common features. This approach allows the output difference between $B$ and $P$ to represent similar but distinct dataset-specific features. This enables the meta-classifier to better distinguish between independent benign models fine-tuned on similar datasets and actual stolen models. In general, a higher drop rate leads to greater differences in dataset-specific features but reduces the preservation of common features. As shown in Table \ref{table:drop_rate}, when $\lambda \% = 0\%$, our method fails to provide ownership verification in almost all cases because the benign shadow model learns dataset-specific features that are too similar to those of the victim model. For drop rates of $\lambda \% \in \{ 10\%, 30\%, 50\%, 70\% \} $, our method succeeds in providing effective verification in all cases.

\begin{table*}[!ht]
\centering
\caption{Effects of model structures.}
%\vspace{-0.6em}
\scalebox{0.725}{
\begin{tabular}{c|cc|cc|cc|cc|cc|cc||cc}
\toprule
Model Stealing$\rightarrow$ & \multicolumn{2}{c|}{Direct-copy} &  \multicolumn{2}{c|}{Fine-tuning}&\multicolumn{2}{c|}{Data Distill} & \multicolumn{2}{c|}{Data-free Distill}  & \multicolumn{2}{c|}{Hard Distill} & \multicolumn{2}{c||}{Soft Distill} & \multicolumn{2}{c}{Independent} \\
\hline
Model Structure$\downarrow$& $\Delta \mu$ & p-value &  $\Delta \mu$& p-value 
&$\Delta \mu$& p-value & $\Delta \mu$& p-value & $\Delta \mu$& p-value & $\Delta \mu$& p-value & $\Delta \mu$& p-value \\
\hline
BEiT & 0.24 & $10^{-16}$ & 0.14 & $10^{-3}$ & 0.13 & $10^{-4}$ & 0.01 & $10^{-3}$ & 0.13 & $10^{-3}$ & 0.17 & $10^{-3}$&-0.33 & 0.99\\
ResNet & 0.28 & $10^{-31}$ & 0.28 & $10^{-21}$ & 0.22 & $10^{-20}$ & 0.23 & $10^{-10}$&0.01 & $10^{-3}$&0.24 & $10^{-24}$&-0.05 & 0.29\\
VGG & 0.11 & $10^{-6}$ &0.01 & $10^{-3}$&0.06 & $10^{-6}$&0.02 & $10^{-3}$&0.03 & $10^{-3}$&0.06 & $10^{-3}$&-0.08 & 0.60\\
CLIP & 0.18& $10^{-14}$ &  0.12& $10^{-9}$&0.18& $10^{-12}$ &{0.13}&${10^{-9}}$&{0.05}& ${10^{-4}}$ & {0.14}& ${10^{-8}}$ & -0.03& 0.20\\

\bottomrule
\end{tabular}
}
\label{table:model_structure}
%\vspace{-1.5em}
\end{table*}

\subsubsection{Effects of the Number of Verification Samples}
\texttt{Holmes} requires specifying the number of verification samples from $\mathcal{D}_s$ ($i.e.$, $m$) used in the hypothesis test. Generally, increasing $m$ mitigates the randomness inherent in the sampling process, thereby enhancing the confidence of the verification results. This is probably the main reason why the p-value consistently decreases as $m$ increases. For example, as shown in Table \ref{table:sample_number}, the p-value decreases from $10^{-6}$ ($m=50$) to $10^{-10}$ ($m=200$) for the soft distillation attack. However, larger $m$ introduces higher computational costs due to increased model queries and processing overhead. This creates a trade-off between verification confidence and computational efficiency, requiring defenders to balance $m$ based on specific deployment constraints and verification requirements.

\subsubsection{Effects of Model Structures}
In this section, we discuss the generalization of our method across different model architectures. Specifically, we consider two cases: (1) the stolen models have the same structures as the victim model, and (2) the stolen models have different structures from the victim model. In the first case, we conduct experiments on BEiT to BEiT. In the second case, we consider three settings, including ResNet34 to ResNet18, VGG19 to VGG16, and CLIP to ResNet18. The fine-tuning process of victim models and the model stealing settings remain consistent with those described in Section \ref{sec:main_setting}. As shown in Table \ref{table:model_structure}, our method remains effective under all settings without misjudgments.

\subsection{Ablation Study}

There are three key parts contained in our method, including \textbf{(1)} the poisoned shadow model $P$, \textbf{(2)} the benign shadow model $B$, and \textbf{(3)}  the selection strategy for $\mathcal{D}_s$ and $\mathcal{D}_f$. In this section, we evaluate the effectiveness of each component.

\subsubsection{The Effectiveness of the Poisoned Shadow Model}

Our method creates a poisoned shadow model $P$ that retains the common features learned by the victim model while disrupting dataset-specific features. The dataset-specific features are then represented by the output difference between $V$ and $P$. To verify the effectiveness of decoupling the common features using $P$, we conduct an experiment where $P$ is omitted.

Specifically, we modify the training set $\mathcal{D}_c$ of meta-classifier $C$ as follows: 
\begin{equation}
\begin{aligned}
   \mathcal{D}_c = & \left\{\left(V(\bm{x}), +1\right)| (\bm{x}, y) \in \mathcal{D}_s \right\} \cup \\
   & \left\{\left(B(\bm{x}), -1\right)| (\bm{x}, y) \in \mathcal{D}_s \right\},
\end{aligned}
\end{equation}
where $C$ is trained to distinguish between the output vectors of $V$ and $B$ on $\mathcal{D}_s$.

Similar to the verification process described in Section \ref{sec:hypothesis_test}, we also adopt hypothesis test to mitigate the randomness as follows:

\begin{definition} Let $\bm{X}$ be a random variable representing samples from $\mathcal{D}_s$. Define $\mu_{S}$ and $\mu_{B}$ as the posterior probabilities of the events $C(S(\bm{x})) = +1$ and $C(S(\bm{x})) = -1$, respectively. Given the null hypothesis $H_0: \mu_{B} + \tau = \mu_{S}$ ($H_1: \mu_{B} + \tau < \mu_{S}$) where $\tau \in [0, 1]$, we claim that the suspicious model $S$ is stolen from the victim (with $\tau$-certainty) if and only if $H_0$ is rejected. 
\end{definition}

\begin{table}[!t]
\centering
\vspace{-1.8em}
\caption{The effectiveness of shadow models.}
%\vspace{-0.8em}
\setlength{\tabcolsep}{0.6mm}{
\begin{tabular}{l|cc|cc|cc}
\toprule
 \multirow{2}{*}{Model Stealing$\downarrow$}& \multicolumn{2}{c|}{Ours} & \multicolumn{2}{c|}{w/o $P$} & \multicolumn{2}{c}{w/o $B$} \\
\cmidrule{2-7}
& $\Delta \mu$ & p-value & $\Delta \mu$ & p-value & $\Delta \mu$ & p-value \\
\hline
Direct-copy & 0.18 & $10^{-14}$ & 0.32 & $10^{-32}$ &0.28 & $10^{-24}$ \\
Fine-tuning & 0.12 & $10^{-9}$ & 0.25 & $10^{-30}$ & 0.23 & $10^{-16}$ \\
Data Distill & 0.18 & $10^{-12}$ & 0.32 & $10^{-30}$ & 0.11& $10^{-6}$\\
Data-free Distill & 0.13 & $10^{-9}$ & 0.36 & $10^{-36}$ & 0.14& $10^{-4}$\\
Hard Distill & 0.05 & $10^{-4}$ & \red{-0.18}& \red{0.19}& 0.02& $10^{-3}$\\
Soft Distill & 0.14 & $10^{-8}$ & \red{-0.21}& \red{0.22}& 0.07& $10^{-3}$\\
\hline
Independent & -0.03 & 0.20 & \red{0.11} & \red{$10^{-3}$} & \red{0.07}& \red{$10^{-7}$}\\
\bottomrule
\end{tabular}
}
\label{table:shadow_model}
\vspace{-1.0em}
\end{table}

As shown in Table \ref{table:shadow_model}, the absence of a poisoned shadow model not only leads to the false positive of the independent model but also causes misjudgment of hard distillation and soft distillation attacks. This is likely because the verification process fails to capture dataset-specific features. Without a poisoned shadow model to decouple these common features, the meta-classifier becomes unable to distinguish between independent models and actual stolen models, resulting in model ownership verification failures.

\begin{table}[!t]
\centering
\vspace{-2.0em}
\caption{Effects of $\mathcal{D}_s$ selection.}
%\vspace{-0.7em}
\setlength{\tabcolsep}{0.4mm}{
\begin{tabular}{l|cc|cc|cc}
\toprule
$\mathcal{D}_{s}\rightarrow$ & \multicolumn{2}{c|}{Lowest Loss} & \multicolumn{2}{c|}{Highest Loss} & \multicolumn{2}{c}{Random} \\
\hline
Model Stealing$\downarrow$& $\Delta \mu$ & p-value & $\Delta \mu$ & p-value & $\Delta \mu$ & p-value \\
\hline
Direct-copy & \textbf{0.18} & $\mathbf{10^{-14}}$ & 0.06 & $10^{-6}$ & 0.20 & $10^{-5}$ \\
Fine-tuning & \textbf{0.12} &$\mathbf{10^{-9}}$ & \red{-0.04} & \red{0.37} & 0.03 & $10^{-4}$ \\
Data Distill & \textbf{0.18} & $\mathbf{10^{-12}}$ & \red{-0.04} & \red{0.43} & 0.10 & $10^{-4}$ \\
Data-free Distill & \textbf{0.13} & $\mathbf{10^{-9}}$ & 0.13 & $10^{-7}$ & 0.07 & $10^{-5}$ \\
Hard Distill & \textbf{0.05} & $\mathbf{10^{-4}}$ & \textbf{0.04} &$\mathbf{10^{-4}}$ & \textbf{0.04} & $\mathbf{10^{-4}}$ \\
Soft Distill & 0.14 & $10^{-8}$ & \textbf{0.02} & $\mathbf{10^{-11}}$ & 0.20 & $10^{-5}$ \\
\hline
Independent & -0.03 & 0.20 & \red{0.18} & \red{$10^{-13}$} & \textbf{-0.05} & \textbf{0.59}\\
\bottomrule
\end{tabular}
}
\label{table:loss_order_ds}
\vspace{-2em}
\end{table}

\subsubsection{The Effectiveness of the Benign Shadow Model}

Our method adopts a benign shadow model $B$ to further enhance the effectiveness of the meta-classifier. To evaluate the effectiveness of $B$, we conduct an experiment where $B$ is omitted. Specifically, we modify the training set $\mathcal{D}_c$ of meta-classifier $C$ as follows: 
\begin{equation}
\begin{aligned}
   \mathcal{D}_c = & \left\{\left(V(\bm{x})-P(\bm{x}), +1\right)| (\bm{x}, y) \in \mathcal{D}_s \right\}.
\end{aligned}
\end{equation}
$C$ is trained only on positive samples using the Isolation Forest method \cite{liu2008isolation}. To mitigate the randomness in the model ownership verification process, we also employ a hypothesis test approach, formulated as follows:

\begin{definition} Let $\bm{X}$ be a random variable representing samples from $\mathcal{D}_s$. Define $\mu_{S}$ and $\mu_{B}$ as the probabilities of the events $C(S(\bm{x})-P(\bm{x})) = +1$ and $C(S(\bm{x})-P(\bm{x})) = -1$, respectively. These probabilities are estimated by the frequencies of +1 and -1 outputs of $C$ over $z$ randomly selected samples from $\mathcal{D}_s$. Given the null hypothesis $H_0: \mu_{B} + \tau = \mu_{S}$ ($H_1: \mu_{B} + \tau < \mu_{S}$) where $\tau \in [0, 1]$, we claim that the suspicious model $S$ is stolen from the victim (with $\tau$-certainty) if and only if $H_0$ is rejected.
\end{definition}

As shown in Table \ref{table:shadow_model}, when $z=10, m=100$, without the benign shadow model, the independent model is misjudged as a stolen model. This is probably because the independent model shares similar dataset-specific features with the victim model. Unable to distinguish the victim's dataset-specific features from the analogous yet distinct ones of other models, the classifier confuses independent models with actual stolen ones.

\subsubsection{The Effectiveness of the Selection Strategy for $\mathcal{D}_s$ and $\mathcal{D}_f$}

Our method selects $\gamma \%$ of samples with the lowest loss values to form $\mathcal{D}_s$ and drops $\lambda \%$ with the lowest loss values to form $\mathcal{D}_f$, as these samples likely correspond to the model's propensity for memorizing dataset-specific patterns rather than common features \cite{ilyas2019adversarial}. To verify the effectiveness of this selection strategy, we conduct experiments using two alternative approaches: selecting the $\gamma \%$ ($\lambda \%$) of samples with the highest loss values and randomly selecting $\gamma \%$ ($\lambda \%$) of the samples.

As shown in Table \ref{table:loss_order_ds}-\ref{table:loss_order_df}, selecting samples with the lowest loss values yields the best verification results in almost all cases for both $\mathcal{D}_s$ and $\mathcal{D}_f$ selections. For $\mathcal{D}_s$, selecting samples with the highest loss values fails in cases of fine-tuning attacks, distillation attacks, and identifying the independent model. This failure is probably because the samples with the highest loss values are less effectively learned by the victim model compared to other samples, making them inadequate to accurately capture dataset-specific features, leading to ineffective verification. For $\mathcal{D}_f$, dropping the highest-loss samples causes failure in identifying the independent model. This is probably because dropping highest-loss samples (instead of dropping lowest-loss ones) makes the benign shadow model retain dataset-specific features too similar to those of the victim model. Consequently, the meta-classifier struggles to distinguish independent models that exhibit similar but distinct dataset-specific features, leading to misjudgment.

\begin{table}[!t]
\centering
\caption{Effects of $\mathcal{D}_f$ selection.}
\vspace{-2em}
\setlength{\tabcolsep}{0.4mm}{
\begin{tabular}{l|cc|cc|cc}
\toprule
$\mathcal{D}_{f}\rightarrow$ & \multicolumn{2}{c|}{Lowest Loss} & \multicolumn{2}{c|}{Highest Loss} & \multicolumn{2}{c}{Random} \\
\hline
Model Stealing$\downarrow$& $\Delta \mu$ & p-value & $\Delta \mu$ & p-value & $\Delta \mu$ & p-value \\
\hline
Direct-copy & \textbf{0.18} & $\mathbf{10^{-14}}$ & 0.15& $10^{-4}$& 0.15& $10^{-9}$\\
Fine-tuning & \textbf{0.12} &$\mathbf{10^{-9}}$ & 0.13& 0.01& 0.09& $10^{-4}$\\
Data Distill & \textbf{0.18} & $\mathbf{10^{-12}}$ & 0.01& $10^{-3}$& 0.11& $10^{-6}$\\
Data-free Distill & \textbf{0.13} & $\mathbf{10^{-9}}$ & 0.08& $10^{-3}$& 0.09& $10^{-5}$\\
Hard Distill & \textbf{0.05} & $\mathbf{10^{-4}}$ & 0.01&0.01& 0.03& 0.01\\
Soft Distill & \textbf{0.14}& $\mathbf{10^{-8}}$ & 0.02& $10^{-3}$& 0.08& $10^{-3}$\\
\hline
Independent & \textbf{-0.03}& \textbf{0.20}& \red{$10^{-3}$}& \red{$10^{-3}$}& -0.01& 0.12\\
\bottomrule
\end{tabular}
}
\label{table:loss_order_df}
\vspace{-2em}
\end{table}

\begin{table*}[!t]
\centering
%\vspace{-0em}
\caption{The resistance to potential adaptive attacks.}
%\vspace{-0.6em}
\scalebox{0.95}{
\begin{tabular}{l|cc|cc|cc|cc}
\toprule
Attack$\rightarrow$ & \multicolumn{2}{c|}{No Attack} &  \multicolumn{2}{c|}{Overwriting}&\multicolumn{2}{c|}{Unlearning}& \multicolumn{2}{c}{Pruning} \\
\hline
\multicolumn{1}{c|}{Model Stealing$\downarrow$} & $\Delta \mu$ & p-value &  $\Delta \mu$ & p-value&$\Delta \mu$ & p-value& $\Delta \mu$ & p-value \\
\hline
Direct-copy & 0.18& $10^{-14}$&  0.11& $10^{-10}$
&0.10& $10^{-8}$& 0.11&$10^{-10}$\\
Fine-tuning & 0.12 & $10^{-9}$ &  0.07& $10^{-3}$&0.02 & 0.01& 0.10& $10^{-4}$\\
Data Distill & 0.18 & $10^{-12}$ &  0.22 & $10^{-7}$
&0.14 & $10^{-7}$ & 0.10&$10^{-8}$\\
Data-free Distill & 0.13 & $10^{-9}$ &  0.04 & $10^{-3}$ 
&0.01 & 0.01& 0.09& 0.01\\
Hard Distill & 0.05 & $10^{-4}$ &  0.03 & $10^{-4}$ 
&0.01 & $10^{-3}$& 0.01& 0.01\\
Soft Distill & 0.14 & $10^{-8}$ &  0.03 & $10^{-4}$
&0.14 & $10^{-4}$& 0.08& $10^{-3}$\\
\bottomrule
\end{tabular}
}
\label{table:unlearning_attack}
%\vspace{-1.5em}
\end{table*}

\section{Discussion}
\subsection{Resistance to Potential Adaptive Attacks}
In this section, we analyze the resistance of our method against potential adaptive attacks. Specifically, we consider overwriting attack, unlearning attack, and model pruning \cite{kwon2022fast} as possible attack strategies, where the adversary is aware of our defense method and has access to $\mathcal{D}_s$ and $\mathcal{D}$.

\subsubsection{Resistance to Overwriting Attack}
The adversary may attempt to overwrite the dataset-specific features inherited from the victim model by implanting alternative features through fine-tuning. Specifically, the adversary fine-tunes the victim model for 20 epochs on the TinyImages dataset (distinct from the victim's CIFAR-10 training data) to assess the erasure of original dataset-specific features. As shown in Table \ref{table:unlearning_attack}, our method remains effective against the overwriting attack. These results demonstrate that even when adversaries are aware of our defense method, our method robustly resists adaptive fine-tuning designed to eliminate dataset-specific features.

\subsubsection{Resistance to Unlearning Attack}

In this scenario, the adversary attempts to make the victim model forget the dataset-specific features, thereby undermining the effectiveness of model ownership verification. Specifically, the adversary randomly reassigns the labels of all samples in $\mathcal{D}_s$ to incorrect labels, creating a modified dataset $\mathcal{D}'_s$. The adversary then fine-tunes the victim model on $\mathcal{D}'_s \cup \mathcal{D}_b$ for 5 epochs, where $\mathcal{D}_b \triangleq \mathcal{D} \backslash \mathcal{D}_s$. As shown in Table \ref{table:unlearning_attack}, our method remains effective in defending against the unlearning attack, demonstrating its robustness.

\subsubsection{Resistance to Model Pruning}
Given the prevalent use of model pruning in adapting personalized LVMs, we evaluate its potential impact on our method's verification reliability. In this experiment, we apply retraining-free pruning \cite{kwon2022fast} to the victim model, following the original paper's settings. As shown in Table \ref{table:unlearning_attack}, our method still succeeds in all cases, demonstrating its robustness against pruning.

\subsection{Potential Task Adaptation}
Arguably, our \texttt{Holmes} framework is inherently adaptable to more complex visual tasks beyond image classification (e.g., image captioning). In this section, we use image captioning as a representative complex generative task to discuss how \texttt{Holmes} can be extended accordingly.

\subsubsection{Method Adaptation}
When adapting to the image captioning task, the core three-stage pipeline of our \texttt{Holmes} (shadow model construction, meta-classifier training, hypothesis testing) remains unchanged, with key modifications limited to adjusting poisoned shadow model construction and projecting text outputs into the feature space.

\vspace{0.3em}
\noindent \textbf{Shadow Model Construction for Image Captioning.} 
In general, shadow model construction follows the same principle in Section \ref{sec:creating_shadow_models}, involving a poisoned shadow model $P_{\text{cap}}$ (preserving common features while disrupting dataset-specific features) and a benign shadow model $B_{\text{cap}}$ (mitigating false positives from independent models). The only adjustment lies in the poisoning process for $P_{\text{cap}}$: Let $\mathcal{D}_{\text{cap}}$ denote the captioning dataset: $\mathcal{D}_{\text{cap}} = \{ (\bm{x}_i, \bm{c}_i) \}_{i=1}^N $, where $\bm{c}_i \in \mathcal{C}$ is the ground-truth caption. $\gamma\%$ of samples with the lowest loss are selected as $\mathcal{D}_{s,\text{cap}}$, which is then poisoned by injecting a trigger $\bm{t}$ via generator $G(\cdot)$ to produce poisoned images $\bm{x}' = G(\bm{x}, \bm{t})$, with a target caption $\bm{c}_t$ ($e.g.$, ``I don't know.'') assigned to form $\mathcal{D}_{p,\text{cap}}$. $P_{\text{cap}}$ is obtained by fine-tuning $V_{\text{cap}}$ on $\mathcal{D}_{p,\text{cap}} \cup \mathcal{D}_{b,\text{cap}} $ ( where $\mathcal{D}_{b,\text{cap}} \triangleq  \mathcal{D}_{\text{cap}} \setminus \mathcal{D}_{s,\text{cap}})$. The benign shadow model $B_{\text{cap}}$ follows the same construction as $B$ in Section \ref{sec:creating_shadow_models}.

\vspace{0.3em}
\noindent \textbf{Feature Projection and Output Difference Calculation.} Since $V_{\text{cap}} $, $P_{\text{cap}}$, and $B_{\text{cap}}$ output text captions (not fixed-dimensional logits), we project these captions into a shared feature space using CLIP's text encoder \cite{radford2021learning}: $\mathcal{E}_{\text{CLIP}}: \mathcal{C} \to \mathbb{R}^d $, where $d$ is the embedding dimension. This encoder leverages pre-trained semantic alignment to ensure consistent comparison of text outputs. For $\bm{x} \in \mathcal{D}_{s,\text{cap}}$, let $\bm{c}_V = V_{\text{cap}}(\bm{x})$, $\bm{c}_P = P_{\text{cap}}(\bm{x}) $, and $\bm{c}_B = B_{\text{cap}}(\bm{x})$ denote captions from the victim, poisoned, and benign models, respectively. These captions are projected into embeddings via $\mathcal{E}_{\text{CLIP}}$: $\bm{e}_V=\mathcal{E}_{\text{CLIP}}(\bm{c}_V)$, $\bm{e}_P=\mathcal{E}_{\text{CLIP}}(\bm{c}_P)$, and $\bm{e}_B=\mathcal{E}_{\text{CLIP}}(\bm{c}_B)$. Output differences are computed as (mirroring Section \ref{sec:method_clf}): $\bm{O}_{v,\text{cap}} = \bm{e}_V - \bm{e}_P$, $\bm{O}_{b,\text{cap}} = \bm{e}_B - \bm{e}_P$, where $\bm{O}_{v,\text{cap}}$ captures $V_{\text{cap}}$'s dataset-specific features, and $\bm{O}_{b,\text{cap}}$ captures unrelated features.

The following process of training ownership meta-classifier and ownership verification follow Sections \ref{sec:method_clf} and \ref{sec:hypothesis_test}: the meta-classifier $C$ is trained on $\bm{O}_{v,\text{cap}}$ (label `+1') and $\bm{O}_{b,\text{cap}}$ (label `-1'). For a suspicious model $S_{\text{cap}}$, project its caption $\bm{c}_S$ to $\bm{e}_S = \mathcal{E}_{\text{CLIP}}(\bm{c}_S)$, compute $\bm{O}_{s,\text{cap}} = \bm{e}_S - \bm{e}_P$, and then apply the hypothesis test on $C(\bm{O}_{s,\text{cap}})$ to verify model ownership.

\subsubsection{Experiment Settings}

\vspace{0.3em}
\noindent \textbf{Dataset Selection.} We conduct experiments on the COCO2014 dataset \cite{lin2014microsoft}, with a training set (16,000 image-caption pairs) and a validation set (8,000 image-caption pairs). For fine-tuning, hard distillation, and soft distillation attacks, 7,000 image-caption pairs are randomly selected from the validation set as the substitute dataset.

\vspace{0.3em}
\noindent \textbf{Settings of Personalization.} The victim models are constructed by fine-tuning the pre-trained foundation models Qwen2.5-VL-7B-Instruct\footnote{\url{https://huggingface.co/Qwen/Qwen2.5-VL-7B-Instruct}} \cite{bai2025qwen2} and Gemma-3-4B-IT\footnote{\url{https://huggingface.co/google/gemma-3-4b-it}} \cite{team2025gemma} on the training set for 5 epochs. The prompt is ``Describe this image''. Low-Rank Adaptation (LoRA) \cite{hu2022lora} is adopted for efficient parameter tuning. Configurations of LoRA are provided in Appendix \ref{appendix:training_configurations_image_captioning}.

\vspace{0.3em}
\noindent \textbf{Settings of Model Stealing.} We simulate the model stealing attacks including: \textbf{(1)} \emph{Direct-copy}: The adversary directly copies the parameters of the victim model; \textbf{(2)} \emph{Fine-tuning}: The adversary obtains the stolen model by fine-tuning the victim model on the substitute dataset for 5 epochs; \textbf{(3)} \emph{Data Distill}: We adopt the Align-KD method \cite{feng2025align} (using its open-source implementation\footnote{\url{https://github.com/fqhank/Align-KD}}) and conduct the distillation attack using the COCO2014 training set. \textbf{(4)} \emph{Hard Distill}: The stolen model is trained on the substitute dataset (with captions generated by the victim model) for 10 epochs. \textbf{(5)} \emph{Soft Distill}: The stolen model is trained on the substitute dataset for 10 epochs, using both captions and visual token representations from the victim model, using the multimodal distillation strategy of LLaVA-KD \cite{cai2025llava}. Regarding model architectures: for \emph{Direct-copy}, \emph{Fine-tuning}, \emph{Data Distill}, and \emph{Soft Distill} attacks, the stolen model adopts the same architecture as the victim model; for \emph{Hard Distill} attacks, when Qwen2.5-VL-7B-Instruct serves as the victim model, the stolen model is Gemma-3-4B-IT, and when Gemma-3-4B-IT is the victim model, the stolen model is Qwen2.5-VL-3B-Instruct \cite{bai2025qwen2}. In particular, we hereby exclude the Data-free Distill attack because, to the best of our knowledge, there exists no data-free knowledge distillation method tailored to image captioning, and directly extending existing techniques is ineffective. Current data-free distillation approaches are primarily developed for image classification \cite{yu2023data, liu2024small, tran2024nayer}, where class information naturally guides the synthesis of category-specific samples. In contrast, image captioning requires generating open-ended textual descriptions without categorical constraints, making these techniques unsuitable. Existing studies on vision–language distillation still depend on structured priors or paired data (\eg, image–caption pairs) to preserve semantic alignment \cite{xuan2023distilling, wei2025open, cai2025llava}. Synthesizing image–caption pairs that align with distillation objectives in a data-free manner remains highly challenging.

Consistent with Section \ref{sec:main_experiments}, we also include an ``Independent" baseline. The independent baseline model is obtained by fine-tuning the foundation model on the COCO2014 training set for 10 epochs. For Qwen2.5-VL-7B-Instruct as the victim, the independent baseline is Gemma-3-4B-IT; for Gemma-3-4B-IT as the victim, the independent baseline is Qwen2.5-VL-7B-Instruct.

\vspace{0.3em}
\noindent \textbf{Settings of Defenses.} For our method, detailed settings align with those in Section \ref{sec:main_experiments}, with task-specific adaptations to image captioning. The poisoned shadow model $P_{\text{cap}}$ is constructed by fine-tuning the victim model $V_{\text{cap}}$ on $\mathcal{D}_{p,\text{cap}} \cup \mathcal{D}_{b, \text{cap}} $ for 5 epochs. We set the modification rate $\gamma \% = 10\%$, resulting in $|\mathcal{D}_{s, \text{cap}}| = 1,600$. We adopt a white square placed at the bottom-right corner of the image as the backdoor trigger, with the corresponding target caption set to ``I don't know''. The benign shadow model $B_{\text{cap}}$ is obtained by fine-tuning the foundation model for 5 epochs on $\mathcal{D}_{f, \text{cap}}$, filtered by dropping $\lambda\% = 50\%$ of the lowest-loss samples from $V_{\text{cap}}$'s training set. The meta-classifier $C$ is trained for 100 epochs with a learning rate of 0.01.

\vspace{0.3em}
\noindent \textbf{Evaluation Metric.} The hypothesis test setting follows Section \ref{sec:main_setting}. Evaluation uses confidence score $\Delta \mu$ and p-value, calculated via hypothesis testing on 100 sampled images ($m=100$). For stolen models, lower p-value and higher $\Delta \mu$ indicate stronger verification; for independent models, higher p-value and lower $\Delta \mu$ reflect better resistance to misclassification.

\begin{table}[!t]
\centering
\vspace{-1.8em}
\caption{Results on image captioning task.}
\setlength{\tabcolsep}{0.75mm}{
\begin{tabular}{l|cc|cc}
\toprule
 \multirow{2}{*}{Model Stealing$\downarrow$}& \multicolumn{2}{c|}{Qwen2.5-VL-7B-Instruct} & \multicolumn{2}{c}{Gemma-3-4B-IT} \\
\cmidrule{2-5}
& $\Delta \mu$ & p-value & $\Delta \mu$ & p-value \\
\hline
Direct-copy & 0.90 &$10^{-36}$&0.89&$10^{-33}$\\
Fine-tuning &0.89&$10^{-20}$&0.81&	$10^{-31}$ \\
Data Distill & 0.89 &$10^{-19}$&0.86&$10^{-20}$  \\
Hard Distill &0.80&$10^{-14}$&0.79&$10^{-13}$  \\
Soft Distill & 0.75	&$10^{-11}$&0.78&$10^{-12}$ \\
\hline
Independent &  -0.99 &1.00&	-0.98&	1.00  \\
\bottomrule
\end{tabular}
}

\label{table:image_caption}
%\vspace{-1.0em}
\end{table}

\subsubsection{Results}

Table \ref{table:image_caption} presents the ownership verification results of \texttt{Holmes} on the COCO2014 image captioning task. For model stealing attacks, all strategies yield high $\Delta \mu$ (0.75–0.90) with p-values well below $\alpha$ ($10^{-11}$ to $10^{-36}$) across both victim models, indicating reliable ownership verification. For independent models, $\Delta \mu$ is negative (-0.98 to -0.99) and p-values are far above $\alpha$, confirming no false positives. These results validate \texttt{Holmes}' effectiveness in image captioning, extending its practicality beyond classification tasks.

\section{Potential Limitations and Future Directions}
\label{sec:potential_limitations}
Although our method demonstrates robust performance against various model stealing attacks (Section \ref{sec:main_experiments}), we acknowledge several limitations that warrant discussion.

\vspace{0.3em}
\noindent \textbf{Indirect Representation of Dataset-Specific Features.} Our method utilizes the label-inconsistent backdoor attack to disrupt the dataset-specific features learned by the victim model. While our experiments show that this approach is effective, we acknowledge that other techniques could be explored for decoupling common features and disrupting dataset-specific ones.

\vspace{0.3em}
\noindent \textbf{Computational Overhead.} Our method involves fine-tuning a poisoned shadow model, a benign shadow model, and training an external meta-classifier, which introduces additional computational overhead. However, training shadow models for reference is a common practice in model watermarking and fingerprinting to evaluate whether these techniques impact model performance \cite{jia2021entangled, maini2021dataset, wang2023free}. In our method, the fine-tuning cost of the shadow models is comparable to that of fine-tuning the victim model itself. For example, in the main experiments in Section \ref{sec:main_experiments}, fine-tuning a shadow model on the CIFAR-10 dataset takes 1 hour and 58 minutes, while fine-tuning on the ImageNet dataset takes 1 hour and 7 minutes, both using an NVIDIA A100 GPU. Besides, this process can be accelerated through parallel GPU execution. The training cost for the meta-classifier is relatively mild, as its structure is lightweight and requires only a limited number of training samples. For example, training it on CIFAR-10 and ImageNet takes approximately 5 minutes on an NVIDIA A100 GPU.

\vspace{0.3em}
\noindent \textbf{Generalization to Diverse Data Domains.} While our method demonstrates effectiveness on standard vision benchmarks (CIFAR-10 and ImageNet), its robustness on specialized domains ($e.g.$, medical images or low-resolution images) remains unverified. The variations in data characteristics of these specialized domains, such as texture patterns or semantic granularity, might affect the decoupling of common and dataset-specific features, warranting further investigation.

\section{Conclusion}
In this paper, we revisited the challenge of verifying ownership for personalized large vision models. We revealed that existing methods, which are typically designed for models trained from scratch, are either ineffective for fine-tuned models, introduce additional security risks, or are prone to misjudgment. To address these issues, we introduced a new fingerprinting paradigm and proposed a harmless model ownership verification method by decoupling common features. By creating a poisoned shadow model (retaining common features but disrupting dataset-specific ones via backdoor attacks) and a benign shadow model (learning distinct dataset-specific features), we distinguished the victim model's dataset-specific features using output differences without altering the victim model or its training process. A meta-classifier and hypothesis test further enhanced verification reliability. We evaluated our method on the CIFAR-10 and ImageNet datasets, and the results demonstrated its effectiveness in detecting various types of model stealing attacks simultaneously.

\section*{Declarations}

\vspace{0.3em}
\noindent \textbf{Data Availability.} The experimental data that support the findings of this study are publicly available, including \href{https://www.cs.toronto.edu/~kriz/cifar.html}{CIFAR-10} \cite{krizhevsky2009learning}, \href{https://image-net.org/index.php}{ImageNet} \cite{deng2009imagenet}, and \href{https://cocodataset.org/#home}{COCO2014} \cite{lin2014microsoft}.

\vspace{0.3em}
\noindent \textbf{Code Availability.} Our codes are available at \url{https://github.com/zlh-thu/Holmes}.

\bibliography{ref}% common bib file
\clearpage
\onecolumn 
\begin{appendices}
\setcounter{theorem}{0}

%\clearpage
%\onecolumn
%\begin{appendices}

\section{Proof of Theorem 1}
\label{appendix:proof_theorem_1}
\setcounter{theorem}{0}
\setcounter{equation}{0}
\setcounter{definition}{0}

We hereby provide the proof of our Theorem \ref{thm1_app}. In general, it is based on the property of $t$-distribution.

\begin{definition} Let $\bm{X}$ be a random variable representing samples from $\mathcal{D}_s$. Define $\mu_{S}$ and $\mu_{B}$ as the posterior probabilities of the events $C(\bm{\bm{O}_s}(\bm{x})) = +1$ and $C(\bm{O}_s(\bm{x})) = -1$, respectively. Given the null hypothesis $H_0: \mu_{B} + \tau = \mu_{S}$ ($H_1: \mu_{B} + \tau < \mu_{S}$) where $\tau \in [0, 1]$, we claim that the suspicious model $S$ is stolen from the victim (with $\tau$-certainty) if and only if $H_0$ is rejected.
\end{definition}

\vspace{0.3em}
\begin{theorem}\label{thm1_app}
Let $\bm{X}$ be a random variable representing samples from $\mathcal{D}_s$. Assume that $\mu_{B} \triangleq \mathbb{P}(C(\bm{O}_s(\bm{X})) = -1) < \beta$. We claim that the verification process can reject the null hypothesis $H_0$ at the significance level $\alpha$ if the identification success rate $R$ of $C$ satisfies that
\begin{equation}
    R > \frac{2(m-1)(\beta+\tau)+t_{1-\alpha}^2 + \sqrt{\Delta} }{2(m-1+t_{1-\alpha}^2)},
\end{equation}
where $\Delta = t_{1-\alpha}^4 + 4t_{1-\alpha}^2(m-1)(\beta+\tau)(1-\beta - \tau)$, $t_{1-\alpha}$ is the $(1-\alpha)$-quantile of $t$-distribution with $(m-1)$ degrees of freedom, and m is the sample size of $\bm{X}$.
\end{theorem}

\begin{proof}

Since $\mu_{B} < \beta$, the original hypotheses $H_0: \mu_{B} + \tau = \mu_{S}$ and $H_1: \mu_{B} + \tau < \mu_{S}$ can be equivalently formulated as:

\begin{gather}
 H_0': \mu_{S} < \beta + \tau, \\
 H_1': \mu_{S} > \beta + \tau.
\end{gather}
Let $E$ indicates the event of whether the suspect model $S$ can be identified by the meta-classifier $C$. As such, 
\begin{equation}
    E \sim \mathcal{B}(1,p),
\end{equation}
where $p  = \Pr(C(\bm{O}_s(\bm{X})) = +1)$ indicates $C$ identify $S$ as a stolen model and $\mathcal{B}$ is the Binomial distribution \cite{hogg2005introduction}.

Let $\bm{x_1}, ..., \bm{x_m}$ denote $m$ samples from $\mathcal{D}_s$ used for ownership verification and $E_1,...,E_m$ denote their prediction events, we know that the identification success of $C$ rate $R$ satisfies
\begin{gather}
 R =  \frac{1}{m} \sum_{i=1}^{m} E_i, \\
 R \sim \frac{1}{m} \mathcal{B}(m,p).
\end{gather}
According to the central limit theorem \cite{hogg2005introduction}, the identification success rate $a$ follows Gaussian distribution $\mathcal{N}(p,\frac{p(1-p)}{m})$ when $m$ is sufficiently large. Similarly, $( R-\beta-\tau)$ also satisfies Gaussian distribution. As such, we can construct the t-statistic as follows
\begin{equation}
    T\triangleq\frac{\sqrt{m}(R-\beta-\tau)}{s}\sim t(m-1),
\end{equation}
where $s$ is the standard deviation of $(R-\beta -\tau )$,
\begin{equation}
\label{eq:std}
    \begin{aligned}
s^2=\frac{1}{m-1}\sum_{i=1}^m(E_i-R)^2=\frac{1}{m-1}(m\cdot R-m\cdot R^{2}).
\end{aligned}
\end{equation}
To reject the hypothesis $H_0'$ at the significance level $\alpha$, we need to ensure that
\begin{equation}
\label{eq:reject_h1}
    \frac{\sqrt{m}(R-\beta-\tau)}{s}>t_{1-\alpha},
\end{equation}
where $t_{1-\alpha}$ is the $(1-\alpha)$-quantile of $t$-distribution with $(m-1)$ degrees of freedom.

According to Equation \ref{eq:std} and Inequality \ref{eq:reject_h1}, we have
\begin{equation}
\label{iq:app_1}
    \sqrt{m-1}\cdot(R-\beta-\tau) - t_{1-\alpha}\cdot\sqrt{R-R^{2}}>0.
\end{equation}

To hold the Inequality \ref{iq:app_1}, two conditions must be satisfied:
\begin{equation}
    R > \beta + \tau,
\end{equation}
and
\begin{equation}
\label{iq:app_2}
    \sqrt{m-1}\cdot(R-\beta-\tau) > t_{1-\alpha}\cdot\sqrt{R-R^{2}}.
\end{equation}

The quadratic inequality of Inequality \ref{iq:app_2} is as: 
\begin{equation}
\label{iq:app_3}
\small (m-1+t_{1-\alpha}^2)R^2 -(2(m-1)(\beta+\tau)+t_{1-\alpha}^2)R + (m-1)(\beta+\tau)^2 > 0.
\end{equation}

For analyzing the valid interval of $R$, we consider the quadratic function as follow:
\begin{equation}
\label{iq:app_4}
\small f(R) = (m-1+t_{1-\alpha}^2)R^2 -(2(m-1)(\beta+\tau)+t_{1-\alpha}^2)R + (m-1)(\beta+\tau)^2.
\end{equation}

The discriminant of $f(R)$ is given by $\Delta = t_{1-\alpha}^4 + 4t_{1-\alpha}^2(m-1)(\beta+\tau)(1-\beta - \tau) > 0$. Thus, the two distinct real roots are given by:
\begin{equation}
\label{iq:app_5}
    R_{1,2} = \frac{2(m-1)(\beta+\tau)+t_{1-\alpha}^2 \pm \sqrt{\Delta} }{2(m-1+t_{1-\alpha}^2)}.
\end{equation}
We can also find that $f(0) = (m-1)(\beta+\tau)^2 > 0$, $f(\beta+\tau) = t_{1-\alpha}^2(\beta+\tau)(\beta+\tau-1)<0$, and $f(1) = (m-1)(\beta+\tau-1)^2$.

According to the intermediate value theorem \cite{andreescu2017intermediate}, since $f(R)$ transitions from positive to negative in $(0, \beta+\tau)$, there must exist a root $R_1$ in this interval. Similarly, since $f(R)$ transitions from negative to positive in $(\beta+\tau, 1)$, there must be a root $R_2$ in this interval. Thus, we have the strict order:
\begin{equation}
\label{iq:app_6}
    0< R_1 < (\beta+\tau) < R_2 <1.
\end{equation}

Because $f(R)$ is positive for $R < R_1$ and $R > R_2$, and negative for $R_1 < R < R_2$, it follows that the Inequality \ref{iq:app_3} is satisfied for $R<R_1$ or $R > R_2$. Given the additional constraint that $R > \beta + \tau$, the only valid solution is
\begin{equation}
\label{iq:app_7}
    R > \frac{2(m-1)(\beta+\tau)+t_{1-\alpha}^2 + \sqrt{\Delta} }{2(m-1+t_{1-\alpha}^2)}.
\end{equation}

\end{proof}

\section{Pseudocode of \texttt{Holmes}}
\label{appendix:pseudocode}
The pseudocode of \texttt{Holmes} is in Algorithm \ref{alg:Holmes}.
\begin{algorithm*}[t]
\caption{The main pipeline of \texttt{Holmes}.}
\label{alg:Holmes}
\begin{algorithmic}[1]
\STATE \textbf{Input:} Victim model $V$, foundation model $F$, suspicious model $S$, training set $\mathcal{D}$, shadow model modification rate $\gamma \%$, shadow model dropping rate $\lambda \%$, target class $y_t$, trigger pattern $\bm{t}$, poisoned image generator $G(\cdot)$, number of samples $m$ for the hypothesis test, and the significance level $\alpha$.
\STATE \textbf{Output:} Ownership verification result for suspicious model $S$.

\STATE \textbf{Step 1: Create Shadow Models}
\STATE \textbf{(a) The Poisoned Shadow Model $P$}:
\STATE Select $\gamma\%$ samples from $\mathcal{D}$ with the lowest loss values to create $\mathcal{D}_s$
\STATE Generate poisoned set: $\mathcal{D}_p = \{ (\bm{x}', y_t) |\bm{x}' = G(\bm{x}; \bm{t}), (\bm{x}, y) \in \mathcal{D}_s \}$
\STATE  Fine-tune the victim model $V$ on $\mathcal{D}_p \cup \mathcal{D}_b$ to create the poisoned shadow model $P$, where $\mathcal{D}_b \triangleq \mathcal{D} \backslash \mathcal{D}_s$
\STATE  \textbf{(b) The Benign Shadow Model $B$}:
\STATE Drop $\lambda\%$ samples from $\mathcal{D}$ with the lowest loss values to create $\mathcal{D}_f$
\STATE Fine-tune the foundation model $F$ on $\mathcal{D}_f$ to create the benign shadow model $B$

\STATE \textbf{Step 2: Train Ownership Meta-Classifier $C_w$}
\STATE Initialize the training set for meta-classifier $\mathcal{D}_c = \emptyset$
    \FOR{each $(\bm{x}, y) \in \mathcal{D}_s$}
        \STATE Compute the output difference between $V$ and $P$: $\bm{O}_v(\bm{x}) = V(\bm{x})-P(\bm{x})$
        \STATE Compute the output difference between $B$ and $P$: $\bm{O}_b(\bm{x}) = B(\bm{x})-P(\bm{x})$
        \STATE $\mathcal{D}_c =  \mathcal{D}_c \cup \{(\bm{O}_v(\bm{x}), +1)\}$
        \STATE $\mathcal{D}_c =  \mathcal{D}_c \cup \{(\bm{O}_b(\bm{x}), -1)\}$
    \ENDFOR
\STATE Train the meta-classifier $C$ on the constructed set $\mathcal{D}_c$
\STATE \textbf{Step 3: Ownership Verification}
    \STATE Randomly sample $m$ samples from $\mathcal{D}_s$
    \STATE Compute the output difference between $S$ and $P$ for these $m$ samples: $\bm{O}_s(\bm{x}) = S(\bm{x})-P(\bm{x})$
    \STATE Conduct single-tailed pair-wise T-test using  $\bm{O}_s(\bm{x})$ and $C$
    \IF{p-value $< \alpha$}
        \STATE Claim $S$ is stolen from $V$
    \ELSE
        \STATE Claim $S$ is not stolen from $V$
    \ENDIF
\end{algorithmic}
\end{algorithm*}

\section{LoRA Configurations of Experiments on Image Captioning}
\label{appendix:training_configurations_image_captioning}

Detailed training configurations of LoRA are provided in Table \ref{tab:lora_params}.

\begin{table}[htbp]
  \centering
  \caption{LoRA Fine-tuning Configuration}
  \begin{tabular}{ll}
    \toprule
    Parameter Name       & Value          \\
    \midrule
    Learning Rate & $10^{-4}$\\
    Warmup Ratio & 0.05 \\
    Gradient Accumulation Steps & 8 \\
    LoRA Rank           & 8 \\
    LoRA Alpha          & 32 \\
    Target Modules      & All Linear Layers \\
    Freeze Vit & True \\
    Training Epochs & 5 \\
    \bottomrule
  \end{tabular}
  \label{tab:lora_params}
\end{table}

\end{appendices}

\end{document}